\documentclass[sn-mathphys,Numbered]{sn-jnl}

\usepackage{amsthm}%
\usepackage[title]{appendix}%

\usepackage{manyfoot}%

\theoremstyle{thmstyleone}%
\newtheorem{proposition}{Proposition}

\theoremstyle{thmstyletwo}%

\theoremstyle{thmstylethree}%
\newtheorem{lemma}{Lemma}
\usepackage{graphicx}

\usepackage{amsmath}
\usepackage{amsfonts}
\usepackage{amssymb}

\usepackage[table]{xcolor}
\usepackage{tikz}
\usepackage{forest}
\usepackage{hf-tikz}
\usetikzlibrary{tikzmark}
\usepackage[ruled, linesnumbered]{algorithm2e}
\usepackage{algpseudocode}
\usepackage{subcaption}
\usepackage{booktabs}
\usepackage{pgfplots}
\usepgfplotslibrary{fillbetween}
\usepgfplotslibrary{groupplots}
\usepackage{filecontents}
\usepackage{bm}
\usepackage{accents}
\usepackage{makecell}
\usepackage{breqn} 

\usepackage{makecell,multirow}
\usepackage{geometry}
\usepackage{pdflscape}

\DeclareMathOperator*{\argmin}{arg\,min}

\SetKwProg{Fn}{Function}{:}{end}
\SetFuncSty{textnormal}
\SetKwComment{Comment}{$\triangleright$\ }{}

\SetCommentSty{mycommfont}

\begin{document}

\title[Hybrid Algorithm Selection and Hyperparameter Tuning on Distributed Machine Learning Resources: A Hierarchical Agent-based Approach]{Hybrid Algorithm Selection and Hyperparameter Tuning on Distributed Machine Learning Resources: A Hierarchical Agent-based Approach}


\author*{\fnm{Ahmad} \sur{Esmaeili}*}\email{aesmaei@purdue.edu}

\author{\fnm{Julia T.} \sur{Rayz}}\email{jtaylor1@purdue.edu}

\author{\fnm{Eric T.} \sur{Matson}}\email{ematson@purdue.edu}

\affil{\orgdiv{Department of Computer and Information Technology}, \orgname{Purdue University}, \orgaddress{\street{401 N. Grant St.}, \city{West Lafayette}, \state{IN, 47907}, \country{USA}}}


\abstract{Algorithm selection and hyperparameter tuning are critical steps in both academic and applied machine learning. On the other hand, these steps are becoming ever increasingly delicate due to the extensive rise in the number, diversity, and distributedness of machine learning resources. Multi-agent systems, when applied to the design of machine learning platforms, bring about several distinctive characteristics such as scalability, flexibility, and robustness, just to name a few. This paper proposes a fully automatic and collaborative agent-based mechanism for selecting distributedly organized machine learning algorithms and simultaneously tuning their hyperparameters. Our method builds upon an existing agent-based hierarchical machine-learning platform and augments its query structure to support the aforementioned functionalities without being limited to specific learning, selection, and tuning mechanisms. We have conducted theoretical assessments, formal verification, and analytical study to demonstrate the correctness, resource utilization, and computational efficiency of our technique. According to the results, our solution is totally correct and exhibits linear time and space complexity in relation to the size of available resources. To provide concrete examples of how the proposed methodologies can effectively adapt and perform across a range of algorithmic options and datasets, we have also conducted a series of experiments using a system comprised of 24 algorithms and 9 datasets.}





\maketitle

\section{Introduction}

The last decades have witnessed a significant surge in the volume and diversity of the Machine Learning (ML) algorithms and datasets provided by multi-disciplinary research communities. Fueled by both the abundance of inexpensive yet powerful computational units and the increasing necessity of using ML-based solutions in every-day applications, such a fast-paced growth has introduced new challenges with respect to the organization, sharing, and selection of ML resources on one hand and automatically optimizing the generalization capability of the learning models on the other. These challenges have become even more compounded by the requirement for efficient parallel approaches to handle locally and geographically distributed ML algorithm and dataset portfolios \cite{esmaeili2020hamlet}.

The ML literature contains numerous research studies focusing on efficient and effective methods for algorithm selection and hyperparameter optimization. The model selection approaches reported in \cite{adankon2009model,banerjee2008model,brazdil2003ranking,chapelle2002model,zeng2017progressive,you2019towards} and the hyperparameter tuning methods proposed in \cite{bardenet2013collaborative,bengio2000gradient,bergstra2011algorithms,bergstra2012random,koch2018autotune,maclaurin2015gradient,parker2020provably}
are some noteworthy examples that address these problems separately. Additionally, there are studies that treat algorithm selection and hyperparameter optimization as a combined problem known as Combined Algorithm Selection and Hyperparameter (CASH) optimization problem. Coining the term ``CASH'', Thornton et al. reported the first prominent research in this area and contributed an automatic tool called Auto-WEKA \cite{thornton2013auto}. Utilizing the full range of classification algorithms provided by WEKA \cite{hall2009weka}, the proposed tool employs the Bayesian Optimization (BO) techniques such as Sequential Model-based Algorithm Configuration (SMAC) \cite{hutter2011sequential} and Tree-structured Parzen Estimator (TPE) \cite{bergstra2011algorithms} to automatically explore the search space comprising the algorithms and their hyperparameters to build models with minimal cross validation error. Hyperopt \cite{bergstra2015hyperopt}, developed on the Scikit-learn ML package \cite{scikit-learn}, extends Auo-WEKA's methodology by including preprocessing alongside algorithm selection and hyperparameter tuning in its BO-based approach.  Feurer et al. have further improved the previous solutions and developed an automatic machine learning (AutoML) tool called Auto-sklearn \cite{feurer2015efficient} that uses Bayesian optimization, warm started with meta-learning, together with automatic ensemble construction to leverage the classifiers found by the optimizer. In \cite{guo2019new}, the authors have explored three variants of the Mixed-integer Parallel Efficient Global Optimization (MiP-EGO) algorithm \cite{van2018automatic} for CASH. These methods involve creating surrogate models for each algorithm, including their hyperparameters, and selecting the best one using an internal optimization method. Another approach is presented in Automatic-AI \cite{czako2021automaticai}, where a modified version of Particle Swarm Optimization (PSO) is employed to explore the search space of algorithms and hyperparameters. To prevent being trapped in local, the proposed method incorporates simulated annealing within each particle during the optimization process. Considering a general definition for agents, the work reported in \cite{cheng2023hpn} uses a shared hyperparameter network to personalize the hyperparameters of the federated learning clients. Specifically, the training of this shared network is centralized and involves using a policy gradient method to solve an optimization problem and then sampling personalized hyperparameters for clients based in their encoding fed to the network.

Multi-Agent Systems (MAS), in its simplest form, refers to a group of autonomous and potentially self-interested intelligent entities that work towards a common goal using their communication and coordination capabilities. Leveraging MAS and agent-based technologies can bring scalability and distribution advantages to ML systems and enable the development of strategic and collaborative learning models \cite{ryzko2020modern,esmaeili2020hamlet}. Numerous research studies have addressed the applications of MAS in the design of ML and data mining solutions.  The works reported in \cite{kargupta1997scalable,kargupta1999collective,gorodetsky2003multi} are among the early efforts in agent-driven data mining tasks, and the studies conducted in \cite{albashiri2008emads,esmaeili2020hamlet} respectively propose general-purpose agent-based data mining and machine learning frameworks. With that said, most of the studies conducted by the community evolve around particular applications \cite{qasem21multi, javadpour2023dmaidps}, and as pointed out in \cite{grislin2022systematic}, the majority of agent abilities employed in agent-driven data mining pertains to goal oriented tasks such as knowledge processing and data collection \cite{chemchem2018deep,yakopcic2019high}. 
Additionally, there are a relatively limited number of studies that utilize an agent-based approach for machine learning algorithm tuning. For instance, in \cite{iranfar2021multi} a multi-agent reinforcement learning is used to optimize the hyperparameters of deep convolutional neural networks. Parker-Holder et al. introduced the Population-based Bandit (PB2) algorithm that utilizes a population of agents to optimize the hyperparameters of a reinforcement learning \cite{parker2020provably}. In \cite{xue2022multi}, Xue et al. have applied mulri-agent reinforcement learning by formulating the dynamic algorithm configuration (DAC) problem as a contextual Markov Decision Process (MDP) and employing the value-decomposition networks algorithm to learn the configuration policy. Finally, in a recent research~\cite{esmaeili2023agent}, an agent-based collaborative random search has been proposed that is based on the generic hierarchical MAS-based hyperparameter tuning model suggested in \cite{esmaeili2022hierarchical}. 

This paper introduces a hybrid agent-based ML algorithm selection and hyperparameter tuning, leveraging the distributed Hierarchical Agent-based Machine LEarning PlaTform (HAMLET) methodology \cite{esmaeili2020hamlet}. To align with the procedures provided in HAMLET, we have augmented its agent selection operators and the learning query structure. Our proposed query structure enables end users to specify both the desired characteristics of the algorithms and the arbitrary set of hyperparameters that should be optimized per each algorithm in a single query. Unlike the similar work we mentioned earlier, the proposed technique is not restricted to a specific category of ML tasks, and thanks to the flexible autonomy provided by its agent-based structure, it allows the integration of various hyperparameter tuning methods.

To better grasp the contribution of this paper, let us consider a scenario where we have a collection of diverse ML algorithms and datasets, developed and maintained across interconnected and potentially geographically distributed devices. To effectively utilize these resources, it is crucial to establish robust procedures to identify, locate, and consolidate them, which can be extensively rigorous depending on their size and distribution. HAMLET platform~\cite{esmaeili2020hamlet} not only facilitates the decentralized organization of ML resources but also offers distributed methods for automatically verifying and executing a batch of ML training and testing tasks using a flexible query structures. This paper extends the capabilities of the HAMLET platform by automating the process of hybrid algorithm selection and hyperparameter tuning on distributed resources, utilizing the same flexible query language provided by HAMLET. As an illustrative example, a user could employ a single query to ``tune the learning rate for all algorithms with a specific hyperparameter and use the optimized models, in combination with other kernel-based algorithms, to select the best model on a given dataset.'' Importantly, the user does not need to possess detailed knowledge about the available resources or their specific location in order to execute the query.

The subsequent sections of this paper are organized as follows: Section~\ref{sec:prelim} presents the preliminaries, offering and overview of the problem and the underlying agent-based structure utilized in our method. In Section~\ref{sec:method}, we delve into the details of the proposed hybrid algorithm selection and tuning method. Theoretical and formal discussions on the performance and correctness verification of our method are presented in  Section~\ref{sec:discussion}. Finally, Section~\ref{sec:conclusion} concludes the paper, summarizing the key findings, and offers suggestions for future work. 

\section{Preliminaries}\label{sec:prelim}
\subsection{Algorithm Selection and hyperparameter Tuning}
Algorithm selection involves determining the best learning model, in terms of generalization performance, from a set of learning algorithms trained on a given dataset. Let $\mathcal{A}$
denote the set of available algorithms and  and $\mathcal{X}$ represent the dataset split into disjoint training and validation sets.  The algorithm selection problem can be formally defined as finding the algorithm $A^*\in \mathcal{A}$ such that:
\begin{equation}\label{eq:AS}
	A^*\in\argmin_{A\in\mathcal{A}}\mathbb{E}_{x\sim\mathcal{G}_x}\left[\mathcal{L}\left(x; A(\mathcal{X}^{(\text{train})})\right)  \right] 
\end{equation}
where $\mathcal{G}_x$ is the grand truth distribution, $\mathcal{L}(x; \mathcal{M}_A)$ is the expected loss of the trained model of algorithm $A$ on independent and identically distributed (i.i.d.) samples $x$, and $\mathbb{E}_{x\sim\mathcal{G}_x}\left[\mathcal{L}\left(x; A(\mathcal{X}^{(\text{train})})\right) \right]$ is the generalization error. In practice, calculating $\mathbb{E}_{x\sim\mathcal{G}_x}$ is not feasible; hence, alternative model evaluation and estimation methods such as stability-based cluster validation \cite{ben2001stability} and cross-validation for classification, regression \cite{kohavi1995study,raschka2018model}, and even clustering \cite{tibshirani2005cluster} tasks are commonly used. Using $k$-fold cross validation technique \cite{kohavi1995study}, equation~\ref{eq:AS} can be rewritten as:
\begin{equation}\label{eq:ASCV}
	A^*\in\argmin_{A\in\mathcal{A}}\frac{1}{k}\sum_{i=1}^{k}\mathcal{L}\left(x\in \mathcal{X}_i^{(\text{valid})}; A(\mathcal{X}_i^{(\text{train})})\right)
\end{equation}
where $\mathcal{X}_i^{(\text{valid})}$ refers to the $i$th partition of the dataset $\mathcal{X}$ that is used for validation and $\mathcal{X}_i^{(\text{train})}=\mathcal{X}\backslash\mathcal{X}_i^{(\text{valid})}$.

The training behavior and the performance of each ML algorithm can be controlled by its hyperparameters. Learning rates and the number of hidden units in a neural network, and the kernel type and regularization penalty of a support vector machine are few examples of such hyperparameters. The objective of hyperparameter tuning/optimization is to select the values within the potentially inter-conditional domains of the hyperparameters such that the generalization error when applied to a specific dataset is minimized. Formally, let $\bm{\lambda}\in\bm{\Lambda}$ denote the hyperparameter vector of a given algorithm $A_{\bm{\lambda}}$. Using $k$-fold cross validation as an estimate for generalization error, the hyperparameter tuning problem can be defined conceptually similar to model selection as follows:
\begin{equation}\label{eq:HPTCV}
	\bm{\lambda}^*\in\argmin_{\bm{\lambda}\in\bm{\Lambda}}\frac{1}{k}\sum_{i=1}^{k}\mathcal{L}\left(x\in \mathcal{X}_i^{(\text{valid})}; A_{\bm{\lambda}}(\mathcal{X}_i^{(\text{train})})\right)
\end{equation}
where all notations are defined the same way as the ones in equation~\ref{eq:ASCV}.

\subsection{Hierarchical Agent-based Machine Learning Platform}
The Hierarchical Agent-based Machine Learning Platform proposed in \cite{esmaeili2020hamlet} adopts a holonic multi-agent systems \cite{fischer2003holonic} approach to organize and coordinate the ML resources in a distributed network based infrastructure. HAMLET serves as an open system that automates resource management, model training, and analysis. It treats each ML algorithm, dataset, and trained model as autonomous entities represented by a self-similar agent architecture. Figure~\ref{fig:hamlet} shows a high-level view of a HAMLET-based networked ML system.  

\begin{figure}
	\centering
	\includegraphics[width=\textwidth]{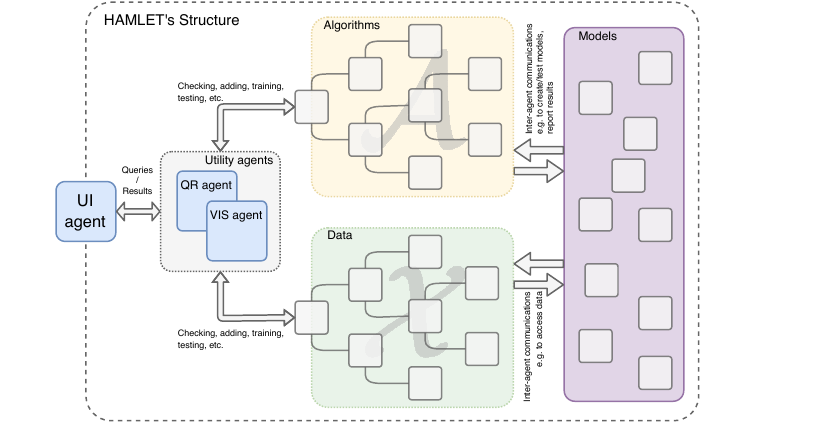}
	\caption{The high level view of a HAMLET-based networked ML system comprising user interface (UI), query/result processing (QR), visualization (VIS), as well as algorithms, data, and model units that are represented using atomic (terminal) and composite (non-terminal) agents.}
	\label{fig:hamlet}
\end{figure}

The distributed organization of ML resources in HAMLET, together with its automated training and testing procedures, rely on the capabilities and skills that are defined in terms of the configuration parameters associated with such resources. Strictly speaking, the configuration of each ML resource and the query specifying the details of an ML task are represented as a parametric set $\mathcal{P}=\{(p_i,v_i)\}_i$ and tuple $\left<\{(a,\mathcal{P}_a)\}_a, \{(x,\mathcal{P}_x)\}_x, \{(o,\mathcal{P}_o)\}_o\right>$ respectively, where $v_i$ is the value of configuration parameter $p_i$, and $\mathcal{P}_a$, $\mathcal{P}_x$, and $\mathcal{P}_o$ respectively denote the configuration parametric sets of algorithm $A_a\in\mathcal{A}$, dataset $X_x\in\mathcal{X}$, and output $o$. A HAMLET-based system begins with a minimal set of components, such as user interface, visualization, and query/result processor agents, and its multi-level structure is formed successively and in a distributed manner as new ML requests are sent to the system over time. HAMLET defines a series of relational and joining operators that are not only used in its builtin organization procedure but also employed by the agents to update their skills and drive their decision making processes in each ML task. To maintain conciseness, we present the definition of the related operations in the context of our proposed methodology in section \ref{sec:method}.  
\section{Methodology}\label{sec:method}
This section delves into the details of the proposed agent-based hybrid algorithm selection and hyperparameter tuning model. We start the presentation of our model by extending the query structure of HAMLET and redefining its corresponding query processing operations to support the intended tasks. Then, using the presented foundations, we provide the details of distributed procedures that are employed by agents to successfully fulfill requested operations.

\subsection{Query Structures}

As stated before, the machine learning queries in HAMLET are characterized by $\left<\{(a,\mathcal{P}_a)\}_a, \{(x,\mathcal{P}_x)\}_x, \{(o,\mathcal{P}_o)\}_o\right>$, where  $a$, $X$, and $o$ respectively denote the identifiers for the ML algorithm, data, and output configuration; and each $\mathcal{P}=\{(p_i,v_i)\}_i$ where $v_i\in \mathbb{D}_{p_i}$. Additionally, HAMLET introduces the new notation, ``$\ast$", as a general placeholder for all the available values for a specific parameter. For instance, an ML testing query specified by $\{(\ast,\{(kernel,rbf)\})\}$ means any ML algorithm with hyperparameter $kernel$ that is set to $rbf$ function, and similarly $\{(svc,\{(kernel,\ast)\})\}$ represents all available $svm$-based classifiers with any value in their $kernel$ hyperparameter. Despite its generic query structure formulations, however, HAMLET's training and testing algorithms work on a single algorithm, i.e., $A_a=(a,\mathcal{P}_a)$, at a time.

This paper not only extends the processing of queries to multiple algorithms in a single query, but also introduces a separate symbol, ``?", to distinguish the parameters that we intend to tune and redefines the available values that a parameter can take as $v_i\in\mathbb{D}_{p_i}\cup\{\ast,?\}$. For example, $\{(svc,\{(kernel,\ast), (C,?)\})\}$ specifies a tuning request for the hyperparameter $C$ in all the available $svm$-based classifiers with any value for their $kernel$, and similarly, $\{(svc,\{(kernel,?)\}),\\ (mlp,\{(lr,?)\})\}$ denotes tuning hyperparameter $kernel$  for all available $svm$-base classifiers and hyperparameter $lr$ (learning rate) for all available multi-layer perceptron algorithms. Please note these queries are more generic and might involve multiple section of the HAMLET structure, depending on its available ML resources. It should also be emphasized that tuning operation is processed only if there is at least one hyperparameter paired with "?" symbol in the query, otherwise, it is treated as a training and the corresponding operations are employed to respond to the query. For example, query with algorithm specification $\{(svc,\{(kernel,rbf)\}), (mlp,\{(lr,?)\})\}$ will only tune the MLP algorithms and returns the tuning result accompanied by the outcomes of training the svm-based classifier.

The selection operation works similar to HAMLET's testing operation in the sense that it locates all the matching resources in the structure and utilizing them to respond to the query. However, unlike testing, selection returns the specifications of the best algorithm based on some performance metrics. As for the selection operation, this paper continues to use HAMLET's existing operators and definitions, but allows for more flexibility in the inclusion of algorithms and incorporating tuning in selection. Before providing the detail procedures, table~\ref{tbl:hamComp} presents an example-based summarization of the extensions by comparing the meaning of a few queries in the existing version of HAMLET\cite{esmaeili2020hamlet} and this paper. Please note that, we use HAMLET's \emph{Testing} query for the sake of comparison, as it is much closer concept to algorithm selection operation, and the query column only include the specification of the algorithms to use. 

\begin{table}[htb]
	\centering
	\caption{Comparison of the meaning of \emph{Testing} query in HAMLET with the implemented selection query in this paper.}
	\label{tbl:hamComp}
	\renewcommand\theadfont{\bfseries}
	\begin{tabular}{p{0.19\linewidth} p{5pt} p{0.35\linewidth} p{5pt} p{0.37\linewidth}}
		\hline\hline
		\thead{Query} & &\thead{Meaning in \cite{esmaeili2020hamlet}\\as a testing query} && \thead{Meaning in this paper\\as a selection query}\\
		\hline
		\makecell[l]{\\$\{(\ast, \{(p_1, v_1)\})\}$} & &Test all algorithms that have hyperparameter $p_1=v_1$ & &Selecting the best among all algorithms that have hyperparameter $p_1=v_1$\\
		\makecell[l]{\\$\{(\ast, \{(p_1, v_1),$\\~~~~~~~$(p_2, v_2)\})\}$}&&Test all algorithms that have hyperparameter $p_1=v_1$ and $p_2=v_2$ & &Selecting the best among all algorithms that have hyperparameter $p_1=v_1$ and $p_2=v_2$ \\
		\makecell[l]{\\$\{(a_1, \{(p_1, v_1)\}),$ \\~~$(a_2, \{(p_2, v_2)\})\}$} && \makecell{\\\textbf{------}$^{\dagger}$} && Selecting the best among all algorithms $a_1$ with hyperparameter $p_1=v_1$ and $a_2$ with hyperparameter $p_2=v_2$\\
		\makecell[l]{\\$\{(a_1, \{(p_1, \ast)\}),$ \\~~$(a_2, \{(p_2, ?)\})\}$} && \makecell{\\\textbf{------}} && Selecting the best among all algorithms $a_1$ with any value for hyperparameter $p_1$ and all algorithms $a_2$ with tuned hyperparameter $p_2$\\
		
		\hline\hline
		\multicolumn{5}{l}{%
			\begin{minipage}{\textwidth}~\\
				\footnotesize{$^\dagger$ HAMLET handles multi-algorithm queries at the query level. That is, the original query is broken down first and then, each travels down the hierarchy and processed by the agents separately.}
			\end{minipage}
		}	
	\end{tabular}
	
\end{table}

Although based on equations~ \ref{eq:ASCV} and \ref{eq:HPTCV}, hyperparameter tuning and algorithm selection share similar processes, this paper treats them in separate but overlapping query types. This is mainly due to the difference in the way that they should be responded; tuning queries usually require a detailed report on candidate hyperparameter values, especially in multi-algorithm hyperparameter optimization, whereas the selection aims at finding a single algorithm matching the criteria.

\subsection{Tuning and Selection}

Similar to HAMLET's training and testing operations, tuning and/or selection comprises two top-down interaction flows in the hierarchy; the first pass determines the access information about the ML resources and handles errors prematurely, and the second pass initiates the corresponding operation. The inter-agent interactions in both algorithm and data hierarchies during the first pass are based on the CNP protocol to locate the data and the most appropriate position to perform tuning and/or selection. To quantize their competence for a query and prepare proposals in the CNP, the agents use a bivariate operator, called \emph{Parametric Similarity Ratio}(PSR) and denoted by $\accentset{\bm{\ast}}{\bm{\sim}}$. To cover the new ``?'' tuning symbol while breaking HAMLET's core processes, this paper redefines PSR as follows: 

\begin{equation}\label{eq:paramsimset}
	\accentset{\ast}{\sim}(P, P^\prime)=\frac{\displaystyle \sum_{\substack{v=v^\prime\\ v=?\\v^\prime=?}}\accentset{\ast}{\sim}((p,v), (p,v^\prime)) + \prod_{v\ne v^\prime}\accentset{\ast}{\sim}((p,v), (p,v^\prime))}{\vert P\vert} 
\end{equation}
where
\begin{equation}\label{eq:paramsimpair}
	\accentset{\ast}{\sim}((p,v), (p,v^\prime))=\begin{cases}
		1 & \text{if } v=v'\\
		\tau & \text{if }  v=?\lor v'=?\\
		\alpha & \text{if } v=\ast\oplus v'=\ast\\
		\beta & \text{otherwise}\\
	\end{cases}
\end{equation}
where $0<\beta<\alpha<\tau<1$; $(p,v)\in P; (p,v^\prime)\in P^\prime$; 
and finally, $\vert\dots\vert$ represents set cardinality.
Moreover, we redefine HAMLET's parametric relational operator $\accentset{\ast}{\leq}$ to support the new tuning symbol ``?'' in coverage checking process. In this redefinition, we have:
\begin{equation}\label{eq:paramIn}
	(p,v)\accentset{\ast}{\leq}(p,v')\iff v\in\{v',\ast,?\} \lor v'\in\{\ast,?\}
\end{equation}
and
\begin{equation}\label{eq:paramSetIn}
	P\accentset{\ast}{\leq}P'\iff \forall (p, v)\in P \exists (p, v')\in P\; \colon\; (p,v)\accentset{\ast}{\leq}(p,v')
\end{equation}
It should be emphasized that all of the redefined operators are compliant with the original version of HAMLET. 

As stated before, the approaches suggested in this paper support multi-algorithm tuning and selection queries. Let $\zeta_{a_i}^g\in\bm{\zeta}^g$ denote the competence quantity of agent $g\in G$ with capability/skill parametric set $P^g$ that is to process algorithm $A_{a_i}$ represented by $(a_i,P_{a_i})$. For a multi-algorithm query with algorithm specifications $\{(a_i,P_{a_i})\}_i$, we define:
\begin{equation}\label{eq:compet1}
	\bm{\zeta}^g=\bigcup_{a_i}\left\{\zeta_{a_i}^g \mid \zeta_{a_i}^g=\accentset{\ast}{\sim}(P_{a_i}, P^g)\right\} 
\end{equation}

Similar to the existing training and testing operations, the proposed tuning and selection procedures also comprise two passes in the depth of the hierarchy: during the first pass we assure that we determine the resources and their addresses to perform the operation, and throughout the second pass the structural changes are applied and the operation is initiated. Algorithm~\ref{alg:processtune} provides the steps for the first pass of the tuning process. In abstract terms, during the travel of query specifications down the hierarchy each agent uses $\accentset{\ast}{\leq}$ operator to determine whether there is a hope for finding a match algorithmic sub-structure in lower levels and forwards an appropriate portion of the query downwards only in the existence of such resources. During the upward travel of the responses from the subordinate agents, on the other hand, each agent selects the best candidate for the query. The candidates characterize the locations where the new tuner structures should be inserted. The intuition for determining the insertion locations is to provide the tuners with the possibility to reuse the existing information in the HAMLET to initiate the optimization process. As shown in line~\ref{ln:processtune_filter} of the algorithm, filtering out the impossible sub-queries is based on both the name of each algorithm and its parametric specifications. On the other hand, lines~\ref{ln:processtune_candids_start} to \ref{ln:processtune_candids_end} of algorithm~\ref{alg:processtune}, details the process of choosing the candidate. According to the process, the current agent running this algorithm is by default the candidate unless there is a unique maximum proposal, checked using the boolean set $\mathcal{F}^{id}$ in line~\ref{ln:processtune_flags}, from its children.  In case it finally selects itself as the candidate, its proposal will be equal to the maximum proposal that it ever received for that specific query. The names chosen in the algorithm are self-explanatory, and the accompanied comments provide extra explanations whenever needed. 

\begin{algorithm}
	\DontPrintSemicolon
	\SetAlgoLined
	\KwIn{$Q^{id}_A=\{(a_i, P_{a_i})\}_{a_i}$\Comment*[f]{the query's algorithm specifications}}
	$\mathcal{B}^{id}\gets\left\{b^{id}_{a_i}=\{me\} \mid (a_i, P_{a_i})\in Q^{id}_A\right\}$\Comment*[f]{{the best candidates}}\label{ln:processtune_initbests}\;
	$\mathcal{R}^{id}\gets \left\{r^{id}_{a_i}=0 \mid (a_i, P_{a_i})\in Q^{id}_A\right\}$\Comment*[f]{{the best proposals}}\;
	$\mathcal{Q}^{id}\gets\emptyset$\Comment*[f]{{possible queries}}\;
	\ForEach(\Comment*[f]{{filtering out impossible queries}}){$(a_i, P_{a_i})\in Q^{id}_A$}{
		\uIf(\Comment*[f]{\emph{Cap} refers to agent's capability set}){$\left((name,\textsf{Name})\boxed{\accentset{\ast}{\leq}}(name,a_i)\;\; \textbf{and}\;\; P_{a_i}\boxed{\accentset{\ast}{\leq}}\textsf{Cap}\right)\;\textbf{or}\; \textsf{Level}=1$}{\label{ln:processtune_filter}
			$\mathcal{Q}^{id}\gets \mathcal{Q}^{id}\bigcup (a_i, P_{a_i})$
		}\Else{
			$r^{id}_{a_i}\gets-1; b_{a_i}^{id}\gets\emptyset$\;
		}
	}
	\uIf(\Comment*[f]{{if no sub-queries are possible}}){$\mathcal{Q}^{id}=\emptyset$}{
		$\textsf{Propose}(\textsf{Parent}, Q^{id}_A, \mathcal{R}^{id})$\;	
	}\Else{
		\uIf(\Comment*[f]{\parbox[t]{.5\textwidth}{a terminal agent in the hierarchy}}){$\textsf{Children}=\emptyset$}{
			$\mathcal{R}^{id}\gets\left\{r^{id}_{a_i}=\zeta_{a_i}^{me}\mid\zeta_{a_i}^{me}=\boxed{\accentset{\ast}{\sim}}(P_{a_i}, \textsf{Cap}),\forall(a_i, P_{a_i})\in \mathcal{Q}^{id}\right\} $\;
			$\textsf{Propose}(\textsf{Parent}, \mathcal{Q}^{id}, \mathcal{R}^{id})$\Comment*[f]{{sends up its calculated proposals}}\;
		}\Else(\Comment*[f]{\parbox[t]{.6\textwidth}{a non-terminal agent in the hierarchy}}){
			$\mathcal{F}^{id}\gets \{f_{a_i}=false\mid \forall(a_i, P_{a_i})\in \mathcal{Q}^{id}\}$\Comment*[f]{to control uniqness}\label{ln:processtune_flags}\; 
			\ForEach{$c \in \textsf{Children}$}{
				$\mathcal{R}_c\gets\textsf{CFP}(c, \mathcal{Q}^{id})$\;
				\ForEach{$r_{a_i}\in\mathcal{R}_c$}{
					\uIf{$r_{a_i} > r^{id}_{a_i}$}{\label{ln:processtune_candids_start}
						$b^{id}_{a_i} \gets \{c\} ;$\quad
						$r^{id}_{a_i}\gets r_{a_i}$\;
					}\ElseIf{$r_{a_i} = r^{id}_{a_i}$}{
						\uIf{$f_{a_i}$}{
							$b^{id}_{a_i} \gets \{me\}$\;
						}\Else{
							$b^{id}_{a_i} \gets \{c\} ;$\quad
							$f_{a_i}\gets true$\;
						}\label{ln:processtune_candids_end}
					}
				}
			}
			\uIf(\Comment*[f]{\parbox[t]{.5\textwidth}{this is the ALG agent}}){$\textsf{Level = 1}$}{
				$\textsf{Inform}(\textsf{Parent}, Q_A^{id}, \mathcal{R}^{id})$\;
			}\Else{
				$\textsf{Propose}(\textsf{Parent}, Q_A^{id}, \mathcal{R}^{id})$\;
			}
		}
	}
	\caption{The steps for filtering out unmatched algorithm specifications ad determining the appropriate tuning location in the hierarchy during the first pass.}
	\label{alg:processtune}
\end{algorithm}

The second pass of the tuning process is more straightforward, as it only involves the candidates chosen during the first pass. The detailed steps are presented in algorithm~\ref{alg:initiatetune}. The unique query id in the \emph{Tune} algorithm is used by the agents as the key to extract internal information about the original query, and $\mathcal{X}_{\emph{info}}$ holds the access information to the validation dataset to be used during the hyperparameter optimization. It should be emphasized that this process is following the agent-based organization of ML resources in HAMLET, and the required dataset information is collected during the first pass and using the exact mechanism provided in \cite{esmaeili2020hamlet}. Tuning can benefit from any useful information from the other subordinate agents. To support that, the chosen candidate uses lines~\real{ln:initiatetune-h} and \ref{ln:initiatetune-s} to collect suggestions from its children regarding the values for ``*'' and ``?'' hyperparameters. Then, it integrates them in line~\ref{ln:initiatetune-i} and uses the result in the tuning process. We do not enforce any specific suggestions or integration mechanism in our extension to HAMLET. Agents can, for instance, implement diverse and incentive-based learning methods based on the history of ML tasks and interactions with the subordinates. The extra information obtained by integrating the suggestions can be used, for instance, as the initial points in the hyperparameter search space $\bm{\lambda}\in\bm{\Lambda}$ for methods such as Bayesian Optimization. After the tuning process, the results are collected and reported upwards in the hierarchy until it reaches HAMLET's root and then reported to the end user. The real hyperparameter optimization is initiated through the \emph{LaunchTuning} method in line~\ref{ln:initiatetune-launch} of the algorithm. In its core, this function creates new member agent(s) that implement equation~\ref{eq:HPTCV} using either central or agent-based methods such as \cite{esmaeili2022hierarchical}. Figure~\ref{fig:tune-example} provides a concrete example for the two passes of the tuning process. In this example, it is assumed that  $\beta=0.1; \alpha=0.6; \tau=0.8$. In this example, two bold green arrows represent the communication channel between the tuners and the provided datasets during the optimization process.

\begin{algorithm}[H]
	\DontPrintSemicolon
	\label{alg:initiatetune}
	\SetAlgoLined
	\KwIn{\emph{query\_id}; $\mathcal{X}_{\emph{info}}$\Comment*[f]{query and dataset info}}
	$\mathcal{T}^{id}\gets\emptyset$\Comment*[f]{tuning results}\;
	\ForEach{$(a_i, P_{a_i})\in \mathcal{Q}^{id}$}{
		\uIf(\Comment*[f]{agent is previously chosen for tuning $a_i$}){$b_{a_i}^{id}=me$}{
			$\mathcal{H}^{id}\gets\left\{p_i\mid (p_i,v_i)\in P_{a_i}, v_i\in\{*,?\}\right\}$\label{ln:initiatetune-h}\;
			$\mathcal{U}^{id}\gets\textsf{Ask}\left(\forall c\in\textsf{Children}, \textsf{Suggest},\mathcal{H}^{id},\mathcal{X}_{\emph{info}}\right)$\label{ln:initiatetune-s}\;
			$\mathcal{I}^{id}\gets\textsf{Integrate}(\mathcal{U}^{id})$\label{ln:initiatetune-i}\;
			$\mathcal{T}^{id}\gets \mathcal{T}^{id}\bigcup\textsf{LaunchTuning}\left((a_i, P_{a_i}), \mathcal{I}^{id},  \mathcal{X}_{\emph{info}}\right)$\label{ln:initiatetune-launch}\;
		}\Else{
			$\mathcal{T}^{id}\gets \mathcal{T}^{id}\bigcup\textsf{Ask}\left(b_{a_i}^{id}, \textsf{Tune}, (a_i, P_{a_i}), \mathcal{X}_{\emph{info}}\right)$\;
		}
	}
	$\textsf{Inform}(\textsf{Parent}, \mathcal{T}^{id})$\;
	\caption{The \emph{Tune} algorithm in the second pass that allocates the required resources and starts tuning the hyperparameters based on the collected information from the first pass.}
\end{algorithm}

\begin{figure}
	\centering
	\begin{subfigure}[b]{0.8\textwidth}
		\centering
		\includegraphics[width=\textwidth]{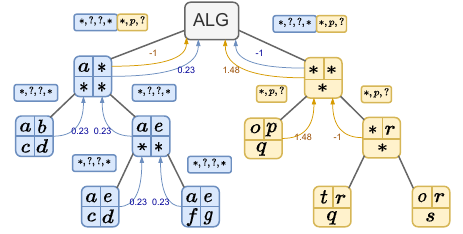}
		\caption{first pass}
		\label{fig:tune-example-fp}
	\end{subfigure}
	\begin{subfigure}[b]{0.8\textwidth}
		\centering
		\includegraphics[width=\textwidth]{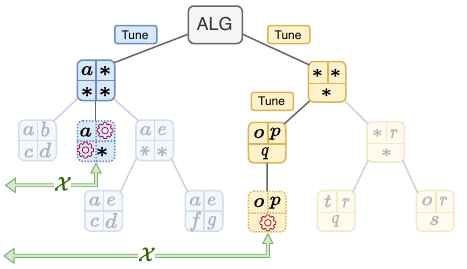}
		\caption{second pass}
		\label{fig:tune-example-sp}
	\end{subfigure}
	\caption{An abstract example that depicts two passes of the proposed tuning algorithm. The used color codes distinguish different algorithm types and the corresponding multi-algorithm tuning query. The shown proposal values computed by the terminal agents are based on the following constants: $\beta=0.1; \alpha=0.6; \tau=0.8$}.
	\label{fig:tune-example}
\end{figure}

As stated before, our contribution allows simultaneous algorithm selection and hyperparameter tuning in its query structure. Combining algorithm selection (equations~\ref{eq:ASCV}) with hyperparameter tuning (equation~\ref{eq:HPTCV}) can be written as:
\begin{equation}\label{eq:CASH}
	A^*_{\bm{\lambda}^*}\in\argmin_{A\in\mathcal{A},\bm{\lambda}\in\bm{\Lambda}}\frac{1}{k}\sum_{i=1}^{k}\mathcal{L}\left(x\in \mathcal{X}_i^{(\text{valid})}; A_{\bm{\lambda}}(\mathcal{X}_i^{(\text{train})})\right)
\end{equation}

This process can be translated into HAMLET's architecture using very similar steps that we had for tuning, i.e., by locating the resources through the first pass and then initiating the validation process and selecting the best algorithms in the second pass. For the sake of brevity, Algorithm~\ref{alg:processselection} only presents the lines that need to be changed in algorithm~\ref{alg:processtune} to support the simultaneous algorithm selection and tuning during the first pass. The changes basically make sure the non-tuning sub-queries are properly reported upwards and the agents candidacies are set accordingly. The additional checks in this algorithm pertain to the intuition that selection among a set of algorithms needs each of such algorithms to already exist or can be built during a tuning process. Moreover, they assure that while the selection covers all the matched capabilities, it stays as close as possible to the terminal agents. Figure~\ref{fig:selection-cases} provides some examples to demonstrate the behavior of the algorithms for various queries. For the sake of simplicity and clarity, we have used a single-algorithm settings for the query.

\begin{algorithm}
	\DontPrintSemicolon
	\SetAlgoLined
	\setcounter{AlgoLine}{18}
	\KwIn{$Q^{id}_A=\{(a_i, P_{a_i})\}_{a_i}$\Comment*[f]{the query's algorithm specifications}}
	\ForEach{$c \in \textsf{Children}$}{\label{ln:processselection-begining}
		$\mathcal{R}_c\gets\textsf{CFP}(c, \mathcal{Q}^{id})$\;
		\ForEach{$r_{a_i}\in\mathcal{R}_c$}{
			$z_{a_i}^{id}\gets \exists (p_i,?)\in P_{a_i}$\label{ln:processselection-ist}\Comment*[f]{determines if a sub-query is tuning}\;
			$l_{a_i}^{id}\gets \exists (p_i,\ast)\in P_{a_i}$\label{ln:processselection-isg}\Comment*[f]{determines if a sub-query is generic}\;
			\uIf{$r_{a_i} > r^{id}_{a_i}$}{
				$b^{id}_{a_i} \gets \{c\} ;$\quad
				$r^{id}_{a_i}\gets r_{a_i}$\;
			}\ElseIf{$r_{a_i} = r^{id}_{a_i}$}{\label{ln:processselection-eq}
				\uIf{$z_{a_i}^{id}\; \textbf{and}\; f_{a_i}$}{
					$b^{id}_{a_i} \gets \{me\}$\;
				}\uElseIf{$z_{a_i}^{id}$}{
					$b^{id}_{a_i} \gets \{c\} ;$\quad
					$f_{a_i}\gets True$\;
				}\uElseIf{$r_{a_i} > 0$}{\label{ln:processselection-positive}
					$b^{id}_{a_i} \gets b^{id}_{a_i} \cup \{c\}$\label{ln:processselection-union}\;
				}\uElseIf{$f_{a_i}\; \textbf{and}\; l_{a_i}^{id}=false$}{
					$b^{id}_{a_i} \gets \{me\}$\;
				}\ElseIf{$l_{a_i}^{id}=false$}{
					$b^{id}_{a_i} \gets \{c\} ;$\quad
					$f_{a_i}\gets True$\;
				}
			}
		}
	}
	\ForEach{$(a_i, P_{a_i})\in \mathcal{Q}^{id}$}{
		\If{$l_{a_i}^{id}\; \textbf{and} z_{a_i}^{id}=false\; \textbf{and}\; r^{id}_{a_i}=0$}{
			$b^{id}_{a_i} \gets \emptyset;$\quad
			$r^{id}_{a_i}\gets -1$\;\label{ln:processselection-end}
		}
	}
	\caption{The changes in the first pass of algorithm~\ref{alg:processtune} to support both selection and tuning queries.}
 \label{alg:processselection}
\end{algorithm}

\begin{figure}
	\centering
	\begin{subfigure}[b]{0.3\textwidth}
		\centering
		\includegraphics[width=\textwidth]{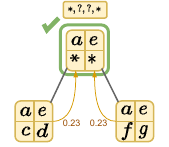}
		\caption{~}
		\label{fig:sca}
	\end{subfigure}
	\begin{subfigure}[b]{0.3\textwidth}
		\centering
		\includegraphics[width=\textwidth]{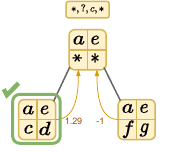}
		\caption{~}
		\label{fig:scb}
	\end{subfigure}
	\begin{subfigure}[b]{0.3\textwidth}
		\centering
		\includegraphics[width=\textwidth]{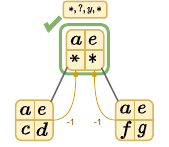}
		\caption{~}
		\label{fig:scc}
	\end{subfigure}
	\begin{subfigure}[b]{0.3\textwidth}
		\centering
		\includegraphics[width=\textwidth]{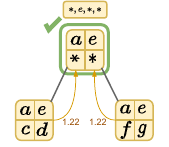}
		\caption{~}
		\label{fig:scd}
	\end{subfigure}
	\begin{subfigure}[b]{0.3\textwidth}
		\centering
		\includegraphics[width=\textwidth]{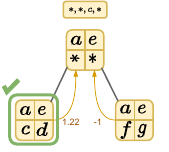}
		\caption{~}
		\label{fig:sce}
	\end{subfigure}
	\begin{subfigure}[b]{0.3\textwidth}
		\centering
		\includegraphics[width=\textwidth]{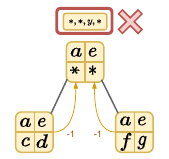}
		\caption{~}
		\label{fig:scf}
	\end{subfigure}
	\caption{Examples demonstrating the behavior of algorithm~\ref{alg:processselection} for various query types. Figures~\ref{fig:sca}, \ref{fig:scb}, and \ref{fig:scc} are for tuning cases, and figures~\ref{fig:scd}, \ref{fig:sce}, and \ref{fig:scf} are for selection cases. In successful passes, the selected candidates are marked in green color, and the shown proposal values are based on the following constants: $\beta=0.1; \alpha=0.6; \tau=0.8$.}
	\label{fig:selection-cases}
\end{figure}

The second pass of the combined selection process behaves similar to the second pass of the tuning process and is presented in algorithm~\ref{alg:selectionsecond}. The \emph{Validate} request in line~\ref{ln:selectionsecond-valid} technically asks the agent to train and validate the algorithm using methods like cross-validation. To help with the understanding the passes of combined selection and tuning, a visual example is provided in figure~\ref{fig:select-example}. To highlight distinctions, the HAMLET structure in this example is kept the same as in our previous example in figure~\ref{fig:tune-example}. In the depicted selection example, we have added a new non-tuning sub-query, and the second pass employs hypothetical loss values to show the flow of reports upwards in the hierarchy. Finally, $w$ value for the third hyperparameter of the second algorithm (represented by yellow color) is assumed to be the tuned value.

\begin{algorithm}[H]
	\DontPrintSemicolon
	\label{alg:selectionsecond}
	\SetAlgoLined
	\KwIn{\emph{query\_id}; $\mathcal{X}_{\emph{info}}$\Comment*[f]{query and dataset info}}
	$\mathcal{S}^{id}\gets\emptyset$\Comment*[f]{selection results: $\mathcal{S}^{id}=\{s_{a_i}^{id}=(\bm{\lambda},\mathcal{L})\}$}\;
	\ForEach{$(a_i, P_{a_i})\in \mathcal{Q}^{id}$}{
		\uIf(\Comment*[f]{agent chosen during the first pass}){$b_{a_i}^{id}=me$}{\label{ln:selectionsecond-base}
			\uIf{$z_{a_i}^{id}$}{
				$\mathcal{H}^{id}\gets\left\{p_i\mid (p_i,v_i)\in P_{a_i}, v_i\in\{*,?\}\right\}$\label{ln:initiatetune-h}\;
				$\mathcal{U}^{id}\gets\textsf{Ask}\left(\forall c\in\textsf{Children}, \textsf{Suggest},\mathcal{H}^{id},\mathcal{X}_{\emph{info}}\right)$\label{ln:selectionsecond-s}\;
				$\mathcal{I}^{id}\gets\textsf{Integrate}(\mathcal{U}^{id})$\label{ln:selectionsecond-i}\;
				$\mathcal{S}^{id}\gets \mathcal{S}^{id}\bigcup\textsf{LaunchTuning}\left((a_i, P_{a_i}), \mathcal{I}^{id},  \mathcal{X}_{\emph{info}}\right)$\label{ln:selectionsecond-launchT}\;
			}
			\Else{
				$\mathcal{S}^{id}\gets\mathcal{S}^{id}\bigcup\textsf{Ask}\left(\forall b\in b_{a_i}^{id}, \textsf{Validate},\mathcal{X}_{\emph{info}}\right)$\label{ln:selectionsecond-valid}\;
			}
		}\ElseIf{$b_{a_i}^{id}\ne\emptyset$}{
			$\mathcal{S}^{id}\gets \mathcal{S}^{id}\bigcup\textsf{Ask}\left(\forall b\in b_{a_i}^{id}, \textsf{Select}, (a_i, P_{a_i}), \mathcal{X}_{\emph{info}}\right)$\label{ln:selectionsecond-select}\;
		}
	}
	\If{$\mathcal{S}^{id}\ne\emptyset$}{
		$\textsf{Inform}(\textsf{Parent}, \argmin\limits_{a_i}\{s_{a_i}^{id}[\mathcal{L}]\})$\;
	}
	\caption{The \emph{Select} algorithm in the second pass that allocating the required resources starts tuning and validations based on the collected information from the first pass.}
\end{algorithm}

\begin{figure}
	\centering
	\begin{subfigure}[b]{0.8\textwidth}
		\centering
		\includegraphics[width=\textwidth]{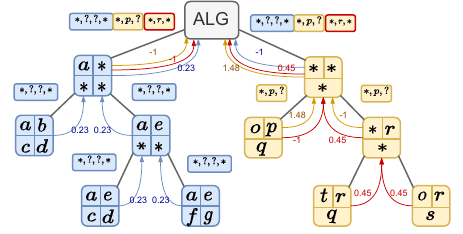}
		\caption{first pass}
		\label{fig:select-example-fp}
	\end{subfigure}
	\begin{subfigure}[b]{0.8\textwidth}
		\centering
		\includegraphics[width=\textwidth]{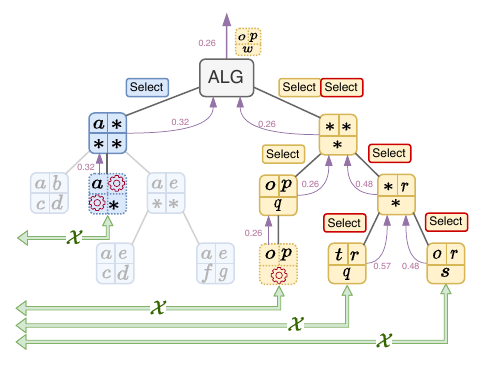}
		\caption{second pass}
		\label{fig:select-example-sp}
	\end{subfigure}
	\caption{An abstract example that depicts two passes of the proposed combined algorithm selection and tuning in HAMLET. The used color codes distinguish different algorithm types and the corresponding multi-algorithm tuning query. The shown proposal values computed by the terminal agents are based on the following constants: $\beta=0.1; \alpha=0.6; \tau=0.8$, and the loss values reported in the second pass are hypothetical.}
	\label{fig:select-example}
\end{figure}

Please note that the presented separate algorithms for each pass of tuning and selection processes are only for clarity purposes and they are merged into a single algorithm for each pass in practice. Furthermore, we assume that the created structures for tuning tasks are ephemeral and shall be released as soon as the process completes. To solidify the results, one can easily follow the resource addition process in \cite{esmaeili2020hamlet} from the tuners parent agent. As the tuned algorithm has already been trained, this process will only be one vertical partial pass in the hierarchy.

\section{Discussions}\label{sec:discussion}
In this section, we delve into the performance details of the proposed methodologies and thoroughly examine their correctness. Formal analysis and verification have been employed for two key reasons: 1) our contribution is not limited to any particular algorithm selection or hyperparameter optimization approach, and 2) relying solely on empirical methods to report and compare results would not adequately evaluate the proposed methodologies or ensure their correctness. Furthermore, this section showcases how the proposed methods are seamlessly integrated into the HAMLET system and demonstrate the expected behavior.

\subsection{Performance}\label{sec:perf_disc}
The contribution of this paper is about the way that agents representing algorithmic and data resources in the HAMLET platform can collaborate to handle a simultaneous algorithm selection and tuning task. 

Let $\mathcal{O}(z)$ and $\mathcal{O}(s)$ denote the time/space complexities of running a tuning method, such as random search, BO, etc., and the selection procedure, including training and validation, respectively. In a structure with $|G|$ number of agents, the worst case scenario happens when all of the terminal agents match a query. Based on the fact that agents run in parallel, the worst condition also requires a maximum possible height for the structure. On the other hand, from the formal discussions provided in \cite{esmaeili2020hamlet}, we know that each non-terminal agent in the structure must have at least two subordinates. Therefore, in the worst case scenario, each level of the HAMLET's structure, except for the root and the last layer, has one terminal and one non-terminal agent, totaling the number of non-terminal agents equal to the height of the structure, $h$. It can be easily shown that $h=\frac{|G|-1}{2}$ and similarly, the total number of terminal agents equals to $\frac{|G|+1}{2}$.

During the first pass of the proposed method, the only conducted operation would be calculating proposal values and comparing the results to each other. Therefore, due to the concurrency of the agents, the time complexity would be $\mathcal{O}(h)=\mathcal{O}(|G|)$. During the second pass, the query flows through the established candidates found in the first pass to reach the terminal agents where the actual ML model building and validation happens. Since, building and validating the models at the terminal agents are conducted in parallel, the time complexity of the second pass would be $\mathcal{O}(|G|+\max_{g}(z_g))$, where $\mathcal{O}(z_g)$ denotes the time complexity of selection/tuning operation at the terminal agent $g$. Since almost all ML operations consume more time than inter-agent communications and basic internal calculations, i.e., $\mathcal{O}(|G|)\ll\mathcal{O}(\max_{g}(z_g))$, it would be safe to conclude that the time complexity of the proposed methodology is $\mathcal{O}(\max_{g}(z_g))$. In other words, the extra time consumed by the platform to handle the query is negligible.

The space complexity of the proposed method depends on the size and characteristics of the data structures employed by each agent internally. There are two key factors: 1) the size of the query that system receives, which is fixed, and 2) the number of subordinates that each agent possesses, which has a reverse relationship with the height of hierarchy, i.e., the deeper the structure, the smaller the average number of subordinates. Taking the space required by all agents into account, the second factor would require $\mathcal{O}(|G|)$ memory. Assuming the space complexity of the selection/tuning operation executed by terminal agent $g$ be $\mathcal{O}(z_g)$, the overall space complexity would be $\mathcal{O}(|G|\times z_g)$.  

Indeed, the worst case that we assumed above will only happen if the entire algorithmic sub-structure of HAMLET is homogeneous, that is, each agent in it represents the same ML algorithm with different hyperparameter values. Since HAMLET is designed to accommodate a diverse set of ML algorithms and datasets, allowing for a wide range of applications and resource configurations, the worst case scenario is very unlikely to happen in realistic scenarios. Moreover, it should be emphasized that the large portion of the required computational and storage power stem from the selection/tuning processes temselves. These tasks inherently involve extensive computation and storage resources due to the nature of exploring different algorithmic configurations and evaluating their performance on large datasets; thus the similar power would be needed even without using HAMLET's platform.

\subsection{Correctness}\label{sec:cor_disc}
As theoretically proved in \cite{esmaeili2020hamlet}, the core functionalities that HAMLET's architecture provides are correct. With that said, we investigate the correctness of our contributions in two main directions: 1) ensuring that the new additions do not disrupt the theoretically proven correctness of HAMLET's functionalities, and 2) verifying the correctness of both tuning and selection mechanisms suggested in the previous sections. This entails examining the proposed methods and providing propositions , lemmas, and theories to support the proofs. Please note that, without any loss of generality, all claims and proofs assume that the algorithms are implemented directly, the utilized HAMLET structure is valid, and no abnormal incidents, such as agent malfunctioning or trustability issues, occur. Additionally, the statements focus on single-algorithm query, but generalizing them to multi-algorithm queries is straightforward due to the separate processing of each algorithm specification within a query.

\begin{proposition}
	The new operator and algorithm selection and/or tuning procedures will not negatively affect HAMLET's core. 
\end{proposition}
\begin{proof}
	There are two ways that this paper could violate the rules and core operations of HAMLET: updating the $\accentset{\ast}{\sim}$ and $\accentset{\ast}{\leq}$ operators and/or altering the structural relationships between the agents. The updates that made to the definition of the aforementioned operators are solely to support the processing of the new tuning character ``?'' and hence, for none-tuning queries, they will yield the same outcome as the original operators in \cite{esmaeili2020hamlet}. On the other hand, for both algorithm selection and tuning procedures we added the new components as an especial and different type of member agents that do are not involved in any training or testing operations. With that said, the existing training and testing procedures can be used with no single update as before.
\end{proof}  

When it comes to the formal verification of the second direction mentioned earlier, it is crucial to establish a precise definition of ``correctness'' and the specific elements we aim to verify. In this paper takes the following definition of correctness into account: our algorithms are correct if they terminate successfully and provide accurate results when they exist, while appropriately indicating their absence.

\begin{lemma}\label{lem:eq}
	If there are more than one agent whose capabilities match a single-algorithm selection/tuning query, their similarity scores are equal.
\end{lemma}
\begin{proof}
	Let $(a_i, P_{a_i})$ characterize the specifications of the selection query. According to algorithm~\ref{alg:processselection}, the similarity ratios are only calculated by the terminal agents, and the non terminal ones only propagate the maximum values upward the hierarchy. Let $g\in G$ and $h\in G$ denote two terminal agents with capabilities $Cap_g$ and $Cap_h$ respectively that match the query, i.e., $P_{a_i}\accentset{\ast}{\leq}Cap_g$ and $P_{a_i}\accentset{\ast}{\leq}Cap_h$. We use proof by contradiction and assume that the similarity scores of these two agents for the selection query are not the same. This implies that the proposal score that one of the agents ($g$ in this proof) is calculating is larger than that of the other agent. We have:
	\begin{align}
		\accentset{\ast}{\sim}(P_{a_i}, Cap_g)&> \accentset{\ast}{\sim}(P_{a_i}, Cap_g)\\
		\sum_{Cap_g, P_{a_i}} + \prod_{Cap_g, P_{a_i}}&> \sum_{Cap_h, P_{a_i}} + \prod_{Cap_h, P_{a_i}} 
	\end{align}
	which implies that $\sum_{Cap_g, P_{a_i}} >  \sum_{Cap_h, P_{a_i}}$ and/or $\prod_{Cap_g, P_{a_i}} >  \prod_{Cap_h, P_{a_i}}$. The first case means:
	\begin{align}
		\exists (p,v)\in P_{a_i} :&\quad (p,v)\in Cap_g, (p,v)\not\in Cap_h\\\Rightarrow&\quad P_{a_i}\accentset{\ast}{\le}Cap_g ,  P_{a_i}\accentset{\ast}{\nleq}Cap_h
	\end{align}
	which contradicts our assumptions. In the second case, we would have $\sum_{Cap_g, P_{a_i}} =  \sum_{Cap_h, P_{a_i}}$ and: 
	\begin{equation}
		\prod_{Cap_g, P_{a_i}} >  \prod_{Cap_h, P_{a_i}}\Rightarrow \alpha^{a_g}\beta^{b_g}\gamma^{c_g}>\alpha^{a_h}\beta^{b_h}\gamma^{c_h}
	\end{equation}
	where $a_g, b_g, c_g, a_h, b_h, c_h\in\mathbb{Z}^{+}$ and $a_g+b_g+c_g=a_h+b_h+c_h$ due to using the same $P_{a_i}$ for each similarity score. According to the assumptions of this lemma, we have: 
	\begin{align}
		P_{a_i}\accentset{\ast}{\leq}Cap_g ,&  P_{a_i}\accentset{\ast}{\leq}Cap_h \\ \Rightarrow \forall(p,v)\in P_{a_i}, (p,v')\in Cap_g, (p,v'')&\in Cap_h :(p,v)\accentset{\ast}{\leq}(p,v'), (p,v)\accentset{\ast}{\leq}(p,v'')\\ \Rightarrow b_g=b_h=0\Rightarrow& a_g+c_g=a_h+c_h \\\alpha^{a_g}\gamma^{c_g}>\alpha^{a_h}\gamma^{c_h}\Rightarrow \alpha^{a_g-a_h}&>\gamma^{c_h-c_g}\Rightarrow \alpha^{a_g-a_h} > \gamma^{a_g-a_h}
	\end{align}
	which is not possible due to the definition of $0<\beta<\alpha<\gamma<1$ in equation~\ref{eq:paramsimpair}. 
\end{proof}

\begin{proposition}
	For any single-algorithm tuning query, there is exactly one agent that is going to be selected for the process.
\end{proposition}
\begin{proof}
	Let $g\in G$ and $h\in G$ be two agents (either terminal or non-terminal) that can be both selected to handle the query. According to lemma~\ref{lem:eq}, the proposals that these two agent will make to their superordinates will be equal. On the other hand, due to the connectedness of HAMLET's hierarchical structure, there will be a common agent $d\in G$, such that $level(d)>level(g)$ and $level(d)>level(h)$ --- we assume that the root of the hierarchy is at level 0. This means that the proposals made/transferred by agents $g$ and $h$ will finally reach agent $d$. No matter in what order the proposals reach agent $d$, the flag variable $f_{a_i}$ in algorithm~\ref{alg:processselection} will set to $True$ and causes agent $d$ to be finally selected as the candidate handling the query. In other words, there will be always a single agent chosen regardless of the number of initial matched agents.
\end{proof}

\begin{lemma}
	Assuming a valid HAMLET architecture and providing a valid query, the proposed method for algorithm selection and tuning is correct. 
\end{lemma}
\begin{proof}
	As for the first part of this proof, we show that if there is a terminal agent that can be used for the query, it will be found through the first pass of the selection/tuning algorithm. Let $g\in G$ be the terminal agent that matches query $(a_i, P_{a_i})$. This implies that $(name, g_{name})\accentset{\ast}{\leq}(name, a_i)$ and $P_{a_i}\accentset{\ast}{\leq}Cap_{g}$. The only reason for agent $g$ to not be found is that at some ancestor agent $d$, the query is not redirected to the branch in which agent $g$ resides. There are potential two reasons for this: 1) agent $d$ does not call for proposals from all of its subordinates, which based on line~\ref{ln:processselection-begining} of algorithm~\ref{alg:processselection} is not possible given that the algorithm is implemented correctly. 2) agent $d$ receives a larger proposal value from another subordinate branch, which based on lemma~\ref{lem:eq} is not possible given that agent $g$ is one of the matches. This proves the partial correctness of the proposed method.
	
	As for the second part towards the total correctness of the method, we show that the proposed selection/tuning process will terminate. Due to using finite sets in each of the \emph{for} loops in  algorithms~\ref{alg:processtune} and \ref{alg:processselection}, the first pass will terminate; as it was discussed above and emphasizing line~\ref{ln:processtune_filter} of algorithm~\ref{alg:processtune}, it will identify the matching agent if there is one, and report otherwise. The second pass is pursued only if there is a potential match for the query in the hierarchy. Based on the recursive nature of this pass, the only potential reason that might cause it not to terminate is the existence of a chain of recursive \emph{Ask} requests (see line~\ref{ln:selectionsecond-select} of algorithm~\ref{alg:selectionsecond}) that never ends in the base condition of line~\ref{ln:selectionsecond-base} in algorithm~\ref{alg:selectionsecond}. Based on the way that the values of variable $b_{a_i}^{id}$ changes during the first pass, we claim that this case will not rise. According to line~\ref{ln:processtune_initbests} of algorithm~\ref{alg:processtune}, this value initially satisfies the base condition, i.e., $b_{a_i}^{id}=me$, and the only times that its value changes are when: 1) the agent, either terminal or non-terminal, does not match the query, 2) the agent is non-terminal and based on a unique maximum proposal that it receives from the subordinates, it sets $b_{a_i}^{id}=g\in G, g\ne me$; and 3) the agent is non-terminal and because of not finding a proper match for a general selection query, it sets $b_{a_i}^{id}=\emptyset$ in line~\ref{ln:processselection-end}. In cases 1 and 3, the as the proposals are the minimum possible values, they will not be selected as the candidates in the upward flow of the first pass, and hence, they will never be examined during the second pass. As with the second case, the agent has some of its subordinates as candidates, and therefore, delegating the query to them, this process continues until a terminal agent for which, $b_{a_i}^{id}=me$, is reached. 
\end{proof}

Apart from correctness of the algorithms, we also strive to guarantee that the selected agent responsible for selection and/or tuning tasks possesses the closest subordinate access to all the necessary algorithmic resources. The following propositions pertain to this property and the abstract structure in figure~\ref{fig:proofexample} is used to help understanding the discussions. In this figure, we assume that agent $g\in G$ is the final agent selected to handle a given selection/tuning query, and set $H$ holds all the terminal agents that match the query. 

\begin{figure}
	\centering
	\includegraphics[width=.5\textwidth]{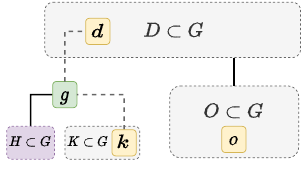}
	\caption{An abstract selection/tuning example in HAMLET. Boxes with solid borders and lower case captions represent individual agents and the ones with dashed border and capital case captions represent subsets of agents in HAMLET. Similarly, direct and indirect inter-agent connections are shown using solid and dashed lines respectively.}
	\label{fig:proofexample}
\end{figure}

\begin{proposition}
	The agent that is finally selected to manage a selection/tuning task is the closest common ancestor of all terminal agents that match the task's query.
\end{proposition}
\begin{proof}
	Assume that $H$ holds all the terminal agents that match the query, and there is an agent, denoted by $d\in D$ at some higher level of the hierarchy selected for the exact same terminal agents $H$ ($d$ will be the farther common ancestor compared to $g$). According to the way that internal variable set $F^id$ is used in algorithm~\ref{alg:processselection} and lemma~\ref{lem:eq} this requires that agent $d$ has at least one other subordinate that proposes the exact same quantity in the CFP process. The existence of such proposal is only feasible if there is another terminal agent $o\in O$ that matches the query. However, this contradicts our original assumption that agent $H$ represents all terminal agents matching the query. 
\end{proof}

\begin{proposition}
	The agent that is finally selected to manage a selection task has access to all required algorithmic resources through its subordinates.
\end{proposition}
\begin{proof}
	An algorithmic resource with hyperparameters $P_{a'_i}$ is required for a selection task if its parameter values fall within the scope of values specified by query $(a_i, P_{a_i})$. Assume terminal agent $k\in K$ among the direct/indirect subordinates of agent $g$ that represent such an algorithmic resource, i.e., $Cap_k=P_{a'_i}$. For this agent's capability to be within the scopes of the query, we need $P_{a_i}\accentset{\ast}{\leq}Cap_k$, and hence, $\accentset{\ast}{\sim}(P_{a_i}, Cap_k)>0$. On the other hand, we know from previous discussions and based on lemma~\ref{lem:eq} we must have $\forall h\in H, \accentset{\ast}{\sim}(P_{a_i}, Cap_k)=\accentset{\ast}{\sim}(P_{a_i}, Cap_k)$. This means that line~\ref{ln:processselection-union} of algorithm~\ref{alg:processselection} is not reachable, which will not happen given that the algorithm is implemented correctly --- the fact that $\forall h\in H$, the expression $\accentset{\ast}{\sim}(P_{a_i}, Cap_k)=\accentset{\ast}{\sim}(P_{a_i}, Cap_k)$ satisfies the condition in line~\ref{ln:processselection-eq} and $\accentset{\ast}{\sim}(P_{a_i}, Cap_k)>0$ satisfies the condition of line~\ref{ln:processselection-positive} in the proposed algorithm.
\end{proof}

\subsection{Empirical Results}\label{sec:empirical}
This section showcases the application of the proposed model on a HAMLET-based distributed ML system composed of 24 ML algorithms (8 algorithms for each of classification, regression, and clustering task types) and 9 datasets (6 classification/clustering data and 3 regression data). For consistency, we have employed the same resource configurations reported in \cite{esmaeili2020hamlet}. Tables \ref{tbl:alg-details} and \ref{tbl:data-details} provides the details of utilized ML algorithms and datasets. The only difference we have applied to the datasets is replacing the \emph{Boston} house prices dataset \cite{harrison1978hedonic} with the \emph{California} house prices data \cite{KELLEYPACE1997291}.

\newcommand{\ra}[1]{\renewcommand{\arraystretch}{#1}}
\begin{sidewaystable}\centering
	\caption{The details of used datasets in the experiment \cite{esmaeili2020hamlet}.}
	\label{tbl:data-details}
	\ra{1.3}
			\begin{tabular}{@{}llcccccc@{}}\toprule
				&\emph{name} & \emph{classes/targets} & \emph{samples per class} & \emph{total samples} & \emph{dimensionality} & \emph{features}\\ \midrule
				\multicolumn{2}{l}{\textbf{classification/clustering:}}\\
				&\textit{Iris} \cite{fisher1936use} & 3 & [50,50,50] & 150 & 4 & real, positive\\
				&\textit{Wine} \cite{lichman2013uci}& 3& [59,71,48]& 178& 13& real, positive\\
				&\textit{Breast cancer} \cite{wolberg1994machine}& 2& [212, 358]& 569& 30& real, positive\\
				&\textit{Digits} \cite{alpaydin1998cascading}\tnote{1}& 10& about 180& 1797& 64& integers [0, 16]\\
				&\textit{Art. Class.}\tnote{2}& 3& [300,300,300]& 900& 20& real (-7.3, 8.9)\tnote{*}\\
				&\textit{Art. Moon}\tnote{3}& 2&[250,250]& 500& 2& real (-1.2, 2.2)\tnote{*}\\
				\multicolumn{2}{l}{\textbf{regression:}}\\
				&\textit{California house prices} \cite{KELLEYPACE1997291}& real [0.15, 5] & --& 20640& 8& real, positive\\
				&\textit{Diabetes} \cite{efron2004least}& integer [25, 346]& --& 442& 10& real (-0.2, 0.2)\\
				&\textit{Art. Regr.}\tnote{4}& real (-488.1, 533.2)& --& 200& 20& real (-4,4)\\
				\bottomrule
			\end{tabular}
				\footnotetext[*]{The values are normalized into [0,1] to prevent potential issues pertaining to negative values in some ML algorithms.}
				\footnotetext[1]{This is a copy of the test set of the UCI ML hand-written digits datasets.}
				\footnotetext[2]{Artificially made using \texttt{make\_classification} function of scikit-learn library \cite{scikit-learn-web}.}
				\footnotetext[3]{Artificially made using \texttt{make\_moons} function of scikit-learn library \cite{scikit-learn-web}.}
				\footnotetext[4]{Artificially made using \texttt{make\_regression} function of scikit-learn library \cite{scikit-learn-web}.}
\end{sidewaystable}

\begin{sidewaystable}\centering
	\caption{The details of used ML algorithms \cite{esmaeili2020hamlet}.}
	\label{tbl:alg-details}
			\begin{tabular}{@{}rrrrcrrrcrrr@{}}\toprule
				& \multicolumn{3}{c}{\emph{classification}} & \phantom{abc}& \multicolumn{3}{c}{\emph{regression}} &
				\phantom{abc} & \multicolumn{3}{c}{\emph{clustering}}\\
				\cmidrule{2-4} \cmidrule{6-8} \cmidrule{10-12}
				& \emph{id} & \emph{name} & \emph{parameters} && \emph{id} & \emph{name} & \emph{parameters} && \emph{id} & \emph{name} & \emph{parameters}\\ \midrule
				& \textbf{A01} & SVC\tnote{1} & kernel=linear && \textbf{A09} & Linear\tnote{6} & defaults && \textbf{A17} & KM\tnote{12} & defaults\tnote{*}\\
				& \textbf{A02}& SVC& kernel=sigmoid&& \textbf{A10}& Ridge\tnote{7}& fit\_inercept=False&& \textbf{A18}& KM& algorithm=full\tnote{*}\\
				& \textbf{A03}& SVC& $\gamma=0.001$&& \textbf{A11}& Ridge& $\alpha=0.5$&& \textbf{A19}& MBKM\tnote{13}& defaults\tnote{*}\\
				& \textbf{A04}& SVC& $C=100,\gamma=0.001$&& \textbf{A12}& KRR\tnote{8}& defaults&& \textbf{A20}& DBSCAN\tnote{14}& defaults\\
				& \textbf{A05}& NuSVC\tnote{2}& defaults&& \textbf{A13}& Lasso\tnote{9}& $\alpha=0.1$&& \textbf{A21}& DBSCAN& metric=cityblock\\
				& \textbf{A06}& ComNB\tnote{3}& defaults&& \textbf{A14}& NuSVR\tnote{10}& defaults&& \textbf{A22}& DBSCAN& metric=cosine\\
				& \textbf{A07}& DTree\tnote{4}& defaults&& \textbf{A15}& NuSVR& $\nu=0.1$&& \textbf{A23}& Birch\tnote{15}& defaults\\
				& \textbf{A08}& NrCent\tnote{5}& defaults&& \textbf{A16}& ElasNet\tnote{11}& defaults&& \textbf{A24}& HAC\tnote{16}& defaults\tnote{*}\\
				\bottomrule
			\end{tabular}
				\footnotetext[1]{\textbf{C-Support Vector Classification}\cite{chang2011libsvm}. Defaults:(C=1.0, kernel='rbf', degree=3, $\gamma$='scale', coef0=0.0, shrinking=True, probability=False, tol=0.001, cache\_size=200, class\_weight=None, verbose=False, max\_iter=-1, decision\_function\_shape='ovr', break\_ties=False)\cite{scikit-learn-web}.}
				\footnotetext[2]{\textbf{Nu-Support Vector Classification}\cite{chang2011libsvm}. Defaults:(nu=0.5, kernel='rbf', degree=3, $\gamma$='scale', coef0=0.0, shrinking=True, probability=False, tol=0.001, cache\_size=200, class\_weight=None, verbose=False, max\_iter=-1, decision\_function\_shape='ovr', break\_ties=False)\cite{scikit-learn-web}.}
				\footnotetext[3]{\textbf{Complement Naive Bayes classifier}\cite{rennie2003tackling}. Defaults:($\alpha$=1.0, fit\_prior=True, class\_prior=None, norm=False)\cite{scikit-learn-web}.}
				\footnotetext[4]{\textbf{Decision Tree Classifier}\cite{hastie2009elements}. Defaults:(criterion='gini', splitter='best', max\_depth=None, min\_samples\_split=2, min\_samples\_leaf=1, min\_weight\_fraction\_leaf=0.0, max\_features=None, random\_state=None, max\_leaf\_nodes=None, min\_impurity\_decrease=0.0, min\_impurity\_split=None, class\_weight=None, presort='deprecated', ccp\_alpha=0.0)\cite{scikit-learn-web}.}
				\footnotetext[5]{\textbf{Nearest Centroid Classifier}\cite{tibshirani2002diagnosis}. Defaults:(metric='euclidean', shrink\_threshold=None)\cite{scikit-learn-web}.}
				\footnotetext[6]{\textbf{Ordinary Least Squares Linear Regression}. Defaults:(fit\_intercept=True, normalize=False, copy\_X=True, n\_jobs=None)\cite{scikit-learn-web}.}
				\footnotetext[7]{\textbf{Ridge Regression}\cite{hoerl1970ridge}. Defaults:($\alpha=1.0$, fit\_intercept=True, normalize=False, copy\_X=True, max\_iter=None, tol=0.001, solver=auto)\cite{scikit-learn-web}.}
				\footnotetext[8]{\textbf{Kernel Ridge Regression}\cite{murphy2012machine}. Defaults: ($\alpha$=1.0, kernel='linear', $\gamma$=None, degree=3, coef0=1, kernel\_params=None)\cite{scikit-learn-web}.}
				\footnotetext[9]{\textbf{Least Absolute Shrinkage and Selection Operator}\cite{tibshirani1996regression}. Defaults:($\alpha$=1.0, fit\_intercept=True, normalize=False, precompute=False, copy\_X=True, max\_iter=1000, tol=0.0001, warm\_start=False, positive=False, random\_state=None, selection='cyclic')\cite{scikit-learn-web}.}
				\footnotetext[10]{\textbf{Nu Support Vector Regression}\cite{chang2011libsvm}. Defaults:($\nu$=0.5, C=1.0, kernel='rbf', degree=3, $\gamma$='scale', coef0=0.0, shrinking=True, tol=0.001, cache\_size=200, verbose=False, max\_iter=-1)\cite{scikit-learn-web}.}
				\footnotetext[11]{\textbf{Elastic Net Regression}\cite{zou2005regularization}. Defaults:($\alpha$=1.0, l1\_ratio=0.5, fit\_intercept=True, normalize=False, precompute=False, max\_iter=1000, copy\_X=True, tol=0.0001, warm\_start=False, positive=False, random\_state=None, selection='cyclic')\cite{scikit-learn-web}.}
				\footnotetext[12]{\textbf{K-Means Clustering}\cite{lloyd1982least}. Defaults:(n\_clusters=8, init='k-means++', n\_init=10, max\_iter=300, tol=0.0001, precompute\_distances='deprecated', verbose=0, random\_state=None, copy\_x=True, n\_jobs='deprecated', algorithm='auto')\cite{scikit-learn-web}.}
				\footnotetext[13]{\textbf{Mini-Batch K-Means Clustering}\cite{sculley2010web}. Defaults:(n\_clusters=8, init='k-means++', max\_iter=100, batch\_size=100, verbose=0, compute\_labels=True, random\_state=None, tol=0.0, max\_no\_improvement=10, init\_size=None, n\_init=3, reassignment\_ratio=0.01)\cite{scikit-learn-web}.}
				\footnotetext[14]{\textbf{Density-Based Spatial Clustering of Applications with Noise}\cite{ester1996density}. Defaults:($\epsilon$=0.5, min\_samples=5, metric='euclidean', metric\_params=None, algorithm='auto', leaf\_size=30, p=None, n\_jobs=None)\cite{scikit-learn-web}.}
				\footnotetext[15]{\textbf{ Birch Clustering}\cite{zhang1996birch}. Defaults:(threshold=0.5, branching\_factor=50, n\_clusters=3, compute\_labels=True, copy=True)\cite{scikit-learn-web}.}
				\footnotetext[16]{\textbf{Hierarchical Agglomerative Clustering}\cite{rokach2005clustering}. Defaults:(n\_clusters=2, affinity='euclidean', memory=None, connectivity=None, compute\_full\_tree='auto', linkage='ward', distance\_threshold=None)\cite{scikit-learn-web}.}
				\footnotetext[*]{The number of clusters is set equal to the number of true classes.}
\end{sidewaystable}

To run the experiment, we added the aforementioned resources to the HAMLET and trained all ML algorithms on all corresponding datasets by simply running a the following query:
\begin{align*}
	\langle \mathcal{A}&=\left\{(\ast,\{\ast\})\right\}, \mathcal{X}=\left\{(\ast,\{(type,\text{train})\})\right\}, \\&O=\left\{type=\text{train}, measures=\{\text{acc}, \text{mse}, \text{fws}\}\right\}\rangle
\end{align*}
where \emph{acc}, \emph{mse}, and \emph{fws} represent accuracy, mean squared error, and Fowlkes-Mallows score \cite{fowlkes1983method} measures, respectively.

Table \ref{tbl:experiments} presents the queries that we sent to the system and the corresponding results obtained. We have used the identifiers from table \ref{tbl:alg-details} to distinguish the winner(s) among the matching candidate that belong the same algorithm family. The arrows in the first column of the table indicate the criteria used for selection/optimization, with $\uparrow$ denoting maximization and $\downarrow$ representing minimization. For instance, $\uparrow$(accuracy) means the corresponding selection/optimization queries was based on maximizing the classification accuracy. Moreover, the hyperparameter optimization technique used for all tuning queries was the hierarchical agent-based method in \cite{esmaeili2023agent} with the following configuration parameter: number of iterations ($\mathcal{I}=10$), budget ($b=3$), slot width ($\mathcal{E}=2^{-6}$), number of candidates ($c=2$), and scaling parameter ($\Delta=\{2,\dots,2\}$). The reported values are the performance on the validation set.

\newgeometry{margin=1cm}
	\begin{sidewaystable}[h]
		\centering
		\caption{The results obtained for algorithm selection and/or tuning queries. The bold face values correspond to the algorithms selection by the algorithm selection process.}
		\label{tbl:experiments}
			\begin{tabular}{ccccccccccc}
				\toprule
				{}& {} & {} & {} & {} & {} & \multicolumn{5}{c}{ML Algorithms} \\\cmidrule(r){7-11}
				{Task} && {ID (fig.)} && {Query Specifications}&&  {SVC} & {NuSVC} & {ComNB} &  {DTree} & {NrCent}\\
				\midrule
				\multirow[c]{5}{*}{\rotatebox[origin=c]{90}{\makecell{classification\\$\uparrow$(accuracy)}}} & {} & 1 (\ref{fig:Q1})& {} &\makecell{
					$\mathcal{A}=\{(\text{SVC},\{\text{(kernel,rbf), ($C$,?), ($\gamma$,?)}\})\}$\\
					$\mathcal{X}=\{(\text{Art. Class.},\{(\text{type, all})\})\}$
				}&&  
				\begin{tabular}{c|c}
					\textbf{0.843} & \makecell{$C=100$\\$\gamma=0.055$}
				\end{tabular} 
				& --- &  --- &  --- &  ---  \\\arrayrulecolor{lightgray}\cmidrule(r){3-11}\arrayrulecolor{black}
				& {} & 2 (\ref{fig:Q0}) & {} &\makecell{
					$\mathcal{A}=\{(\text{*},\{\text{*}\})\}$\\
					$\mathcal{X}=\{(\text{Iris},\{(\text{type, all})\})\}$
				}&&  A01,A04:\textbf{0.980} & \textcolor{darkgray}{0.967} &  \textcolor{darkgray}{0.667} &  \textcolor{darkgray}{0.967} &  \textcolor{darkgray}{0.920} \\\arrayrulecolor{lightgray}\cmidrule(r){3-11}\arrayrulecolor{black}
				& {} & 3 (\ref{fig:Q3}) & {} &\makecell{
					$\mathcal{A}=\{(\text{*},\{\text{(kernel,rbf), ($\gamma$,?)}\})\}$\\
					$\mathcal{X}=\{(\text{Wine},\{(\text{type, all})\})\}$
				}&&  
				\begin{tabular}{c|c}
					\textbf{0.782} & \makecell{$\gamma=0.001$}
				\end{tabular}
				& 
				\textcolor{darkgray}{0.708$|\gamma=0.0004$}
				&  --- & --- & --- \\\midrule\midrule
				{}& {} & {} & {} & {} & {} & {Ridge} & {KRR} & {Lasso} & {NuSVR} & {ElasNet} \\
				\midrule
				\multirow[c]{5}{*}{\rotatebox[origin=c]{90}{\makecell{regression\\$\downarrow$(mse.)}}} & {} & 4 (\ref{fig:Q4}) & {} &\makecell{
					$\mathcal{A}=\{(\text{NuSVR},\{\text{(kernel,rbf), ($C$,?), ($\gamma$,?)}\})\}$\\
					$\mathcal{X}=\{(\text{Art. Regr.},\{(\text{type, all})\})\}$
				}&&  {---} & --- & --- &  
				\begin{tabular}{c|c}
					\textbf{5.037} & \makecell{$C=1000$\\$\gamma=0.002$}
				\end{tabular}
				& ---  \\\arrayrulecolor{lightgray}\cmidrule(r){3-11}\arrayrulecolor{black}
				& {} & 5\tnote{*} (\ref{fig:Q0})& {} &\makecell{
					$\mathcal{A}=\{(\text{*},\{\text{*}\})\}$\\
					$\mathcal{X}=\{(\text{Diabetes},\{(\text{type, all})\})\}$
				}&&  A10:\textbf{27283.229} & A12:\textbf{27283.229} & \textcolor{darkgray}{3008.892} & \textcolor{darkgray}{5016.415} & \textcolor{darkgray}{5930.2} \\\arrayrulecolor{lightgray}\cmidrule(r){3-11}\arrayrulecolor{black}
				& {} & 6 (\ref{fig:Q6}) & {} &\makecell{
					$\mathcal{A}=\{(\text{*},\{\text{($\alpha$,?)}\})\}$\\
					$\mathcal{X}=\{(\text{California},\{(\text{type, all})\})\}$
				}&&  
				\begin{tabular}{c|c}
					\textbf{0.5581} & \makecell{$\alpha=9.977$}
				\end{tabular}  
				& 
				\textcolor{darkgray}{0.644$|\alpha=2.727$}
				& 
				\textcolor{darkgray}{0.585$|\alpha=0.037$} 
				& {---} &
				\textcolor{darkgray}{0.5582$|\alpha=0.0004$} 
				\\\midrule\midrule
				{}& {} & {} & {} & {} & {} & {KM} & {MBKM} & {DBSCAN} & {Birch} & {HAC}\\
				\midrule
				\multirow[c]{5}{*}{\rotatebox[origin=c]{90}{\makecell{clustering\\$\uparrow$(fms)}}} & {} & 7 (\ref{fig:Q7})& {} &\makecell{
					$\mathcal{A}=\{(\text{DBSCAN},\{\text{($\epsilon$,?), (metric,?)}\})\}$\\
					$\mathcal{X}=\{(\text{Art. Class.},\{(\text{type, all})\})\}$
				}&&  --- & --- & 
				\begin{tabular}{c|c}
					\textbf{0.576} & \makecell{$\epsilon=1.534$\\$metric=l1$}
				\end{tabular} 
				& --- & ---  \\\arrayrulecolor{lightgray}\cmidrule(r){3-11}\arrayrulecolor{black}
				& {} & 8 (\ref{fig:Q0}) & {} &\makecell{
					$\mathcal{A}=\{(\text{*},\{\text{*}\})\}$\\
					$\mathcal{X}=\{(\text{Art. Class.},\{(\text{type, all})\})\}$
				}&&  \textcolor{darkgray}{A17:0.475} & \textcolor{darkgray}{0.488} &  A20,A21:\textbf{0.575} & \textcolor{darkgray}{0.496} & \textcolor{darkgray}{0.543} \\\arrayrulecolor{lightgray}\cmidrule(r){3-11}\arrayrulecolor{black}
				& {} & 9 (\ref{fig:Q9}) & {} &\makecell{
					$\mathcal{A}=\{(\text{*},\{\text{(n\_clusters,?)}\})\}$\\
					$\mathcal{X}=\{(\text{Art. Class.},\{(\text{type, all})\})\}$
				}&& 
				\textcolor{darkgray}{0.658$|n\_clusters=3$} 
				& 
				\begin{tabular}{c|c}
					\textbf{0.659} & \makecell{$n\_clusters=3$}
				\end{tabular}
				& -- & 
				\textcolor{darkgray}{0.576$|n\_clusters=1$}
				& \textcolor{darkgray}{0.576$|n\_clusters=1$}
				\\
				\bottomrule
			\end{tabular}
				\footnotetext[*]{ Due to space restrictions and the fact that the linear regression algorithm (A09) is only involved in this particular query, its results have been omitted from the table. However, for the purpose of comparison, it is worth noting that the mean squared error (MSE) of this model is reported as 2993.073.}
	\end{sidewaystable}
\restoregeometry

The system's structure, consisting of agents representing the ML algorithms described in Table \ref{tbl:alg-details}, is automatically captured and visualized using HAMLET's VIZ utility agent. Figure \ref{fig:queries} illustrates the snapshot of the system structure, focusing specifically on the algorithm section for clarity. Figure \ref{fig:Q0} depicts the initial state of the system prior to initiating the queries mentioned in Table \ref{tbl:experiments}. This state will remain unchanged for queries 2, 5, and 8 since they only involve selecting from the available resources. It is important to note that all tuning results are assumed to be ephemeral, resulting in structural changes once a tuning process concludes. 

\begin{figure}
	\centering
	\begin{subfigure}[b]{0.45\textwidth}
		\centering
		\includegraphics[width=\textwidth]{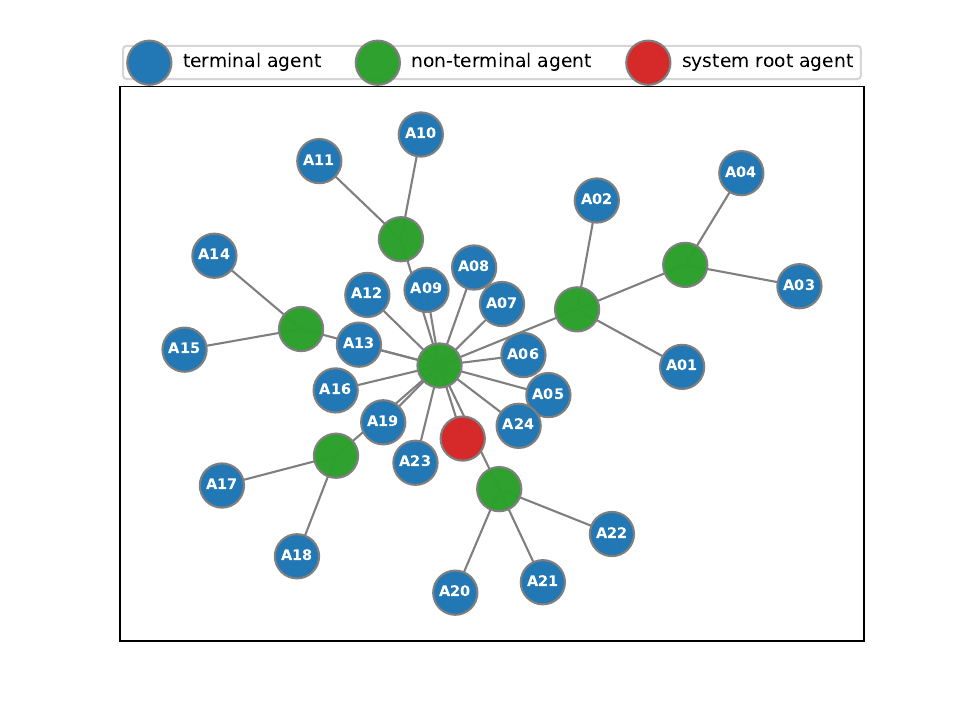}
		\vspace{-2.5em}
		\caption{Queries 2, 5, and 8}
		\label{fig:Q0}
	\end{subfigure}
	\begin{subfigure}[b]{0.45\textwidth}
		\centering
		\includegraphics[width=\textwidth]{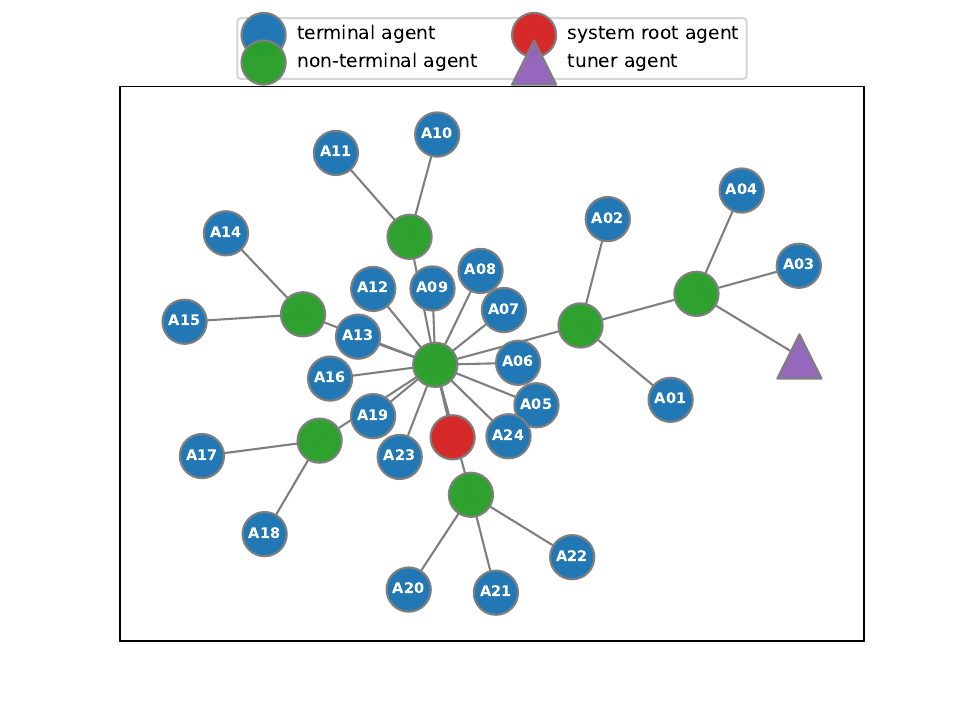}
		\vspace{-2.5em}
		\caption{Query 1}
		\label{fig:Q1}
	\end{subfigure}\\\vspace{1em}
	\begin{subfigure}[b]{0.45\textwidth}
		\centering
		\includegraphics[width=\textwidth]{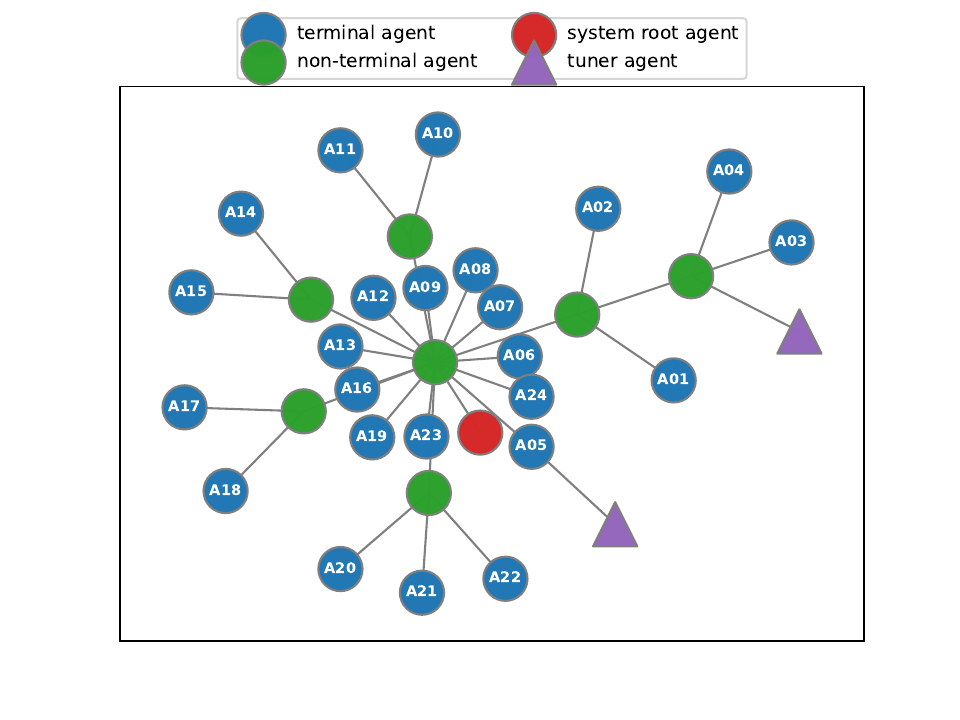}
		\vspace{-2.5em}
		\caption{Query 3}
		\label{fig:Q3}
	\end{subfigure}
	\begin{subfigure}[b]{0.45\textwidth}
		\centering
		\includegraphics[width=\textwidth]{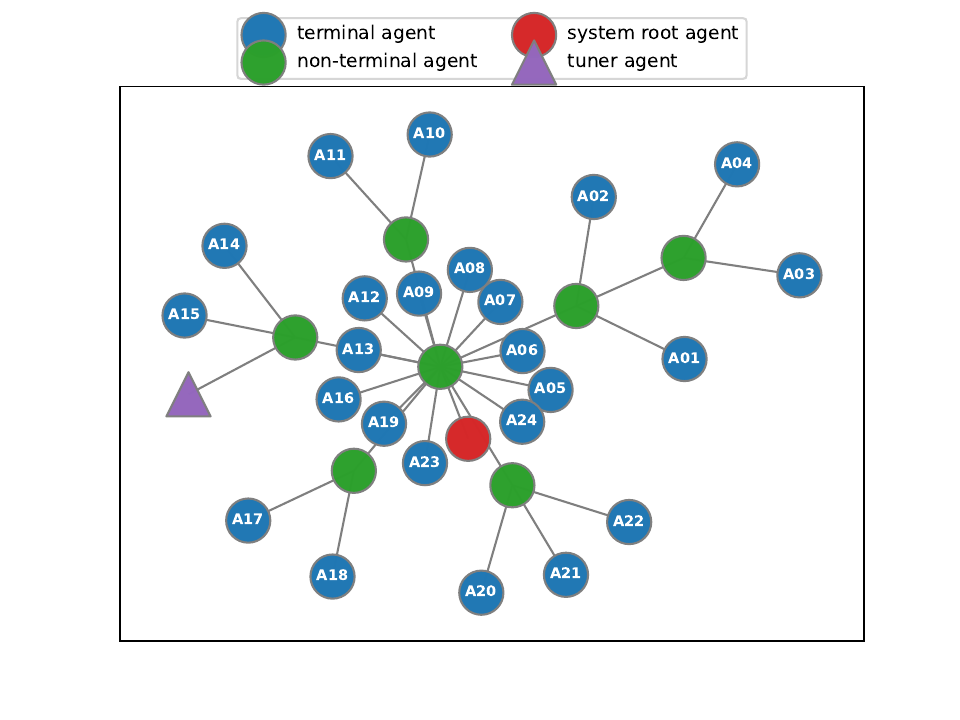}
		\vspace{-2.5em}
		\caption{Query 4}
		\label{fig:Q4}
	\end{subfigure}\\\vspace{1em}
	\begin{subfigure}[b]{0.45\textwidth}
		\centering
		\includegraphics[width=\textwidth]{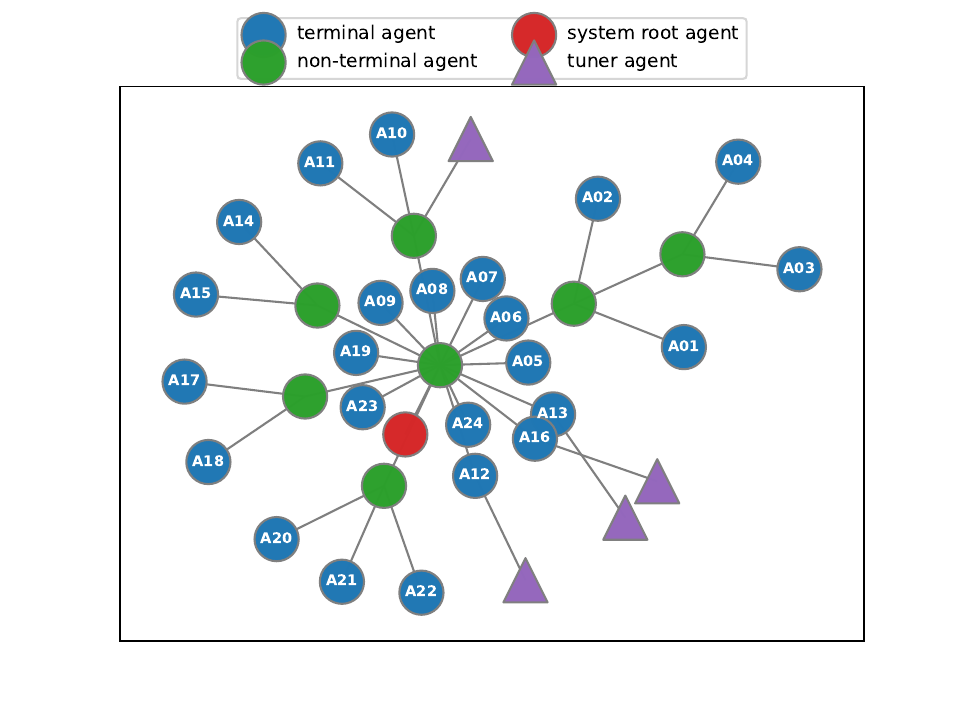}
		\vspace{-2.5em}
		\caption{Query 6}
		\label{fig:Q6}
	\end{subfigure}
	\begin{subfigure}[b]{0.45\textwidth}
		\centering
		\includegraphics[width=\textwidth]{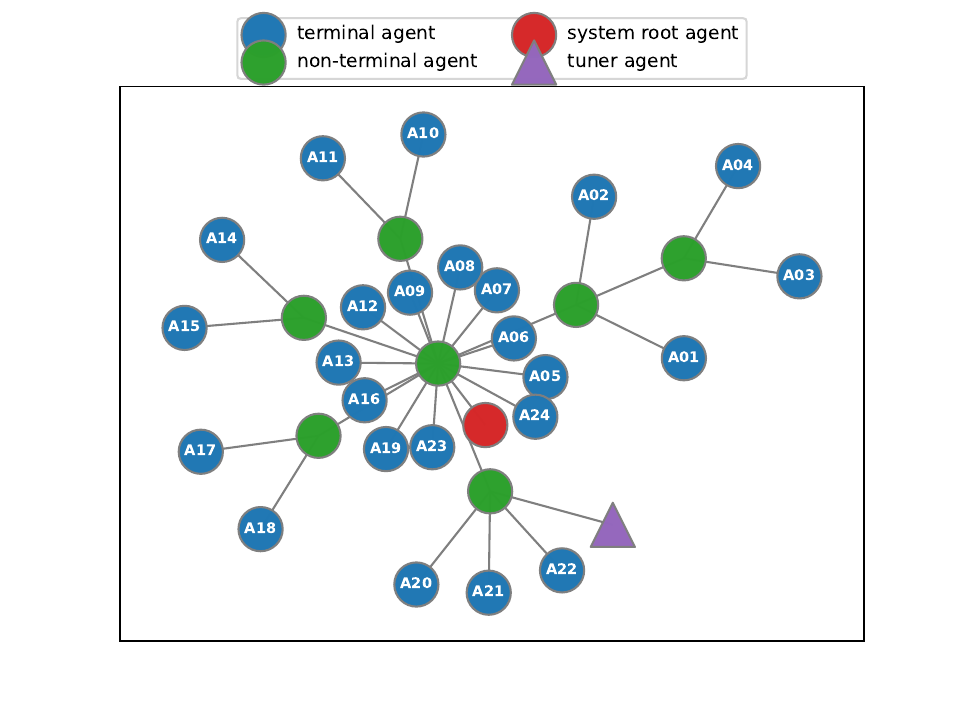}
		\vspace{-2.5em}
		\caption{Query 7}
		\label{fig:Q7}
	\end{subfigure}\\\vspace{1em}
	\begin{subfigure}[b]{0.45\textwidth}
		\centering
		\includegraphics[width=\textwidth]{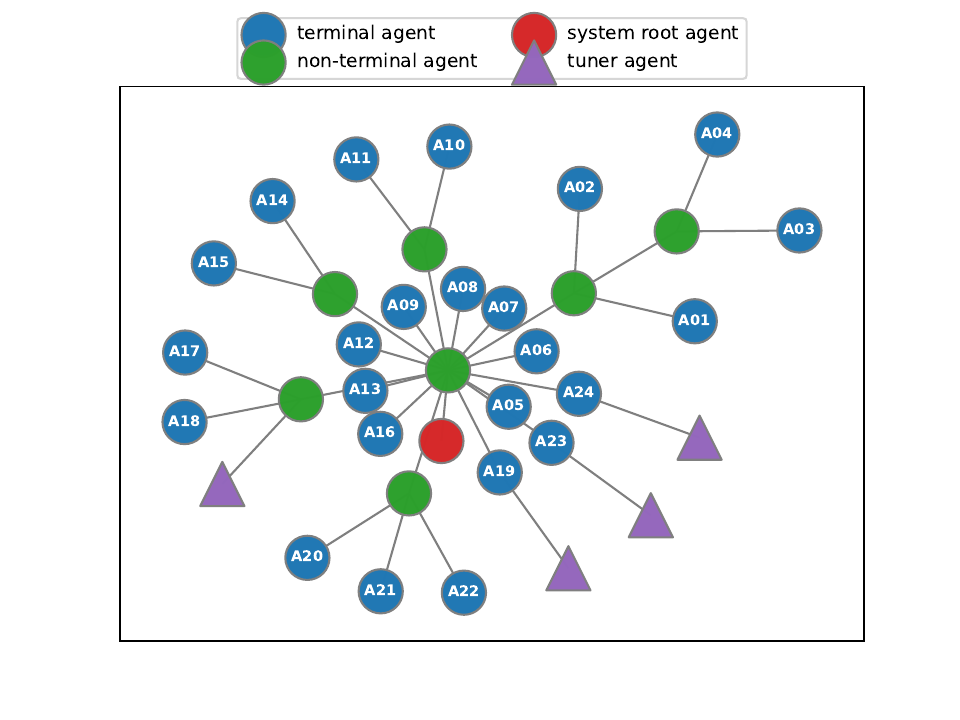}
		\vspace{-2.5em}
		\caption{Query 9}
		\label{fig:Q9}
	\end{subfigure}    
	\caption{The network structure representing the main algorithm agents and their interconnections at the end of table \ref{tbl:experiments} queries.}
	\label{fig:queries}
\end{figure}

To enhance clarity, the results presented in Table \ref{tbl:experiments} have been grouped according to the type of tasks they target. We conducted tuning, selection, and hybrid selection-tuning queries on these task groups. Furthermore, to facilitate comparison, we have included results from non-winning candidates, which are highlighted in gray. The following provides an overview of the queries, indexed by their ID:
\begin{enumerate}
	\item This query focuses on tuning the hyperparameters $c\sim logUniform(10^{-2},10^{13})$ and $\gamma\sim Uniform(0,1)$ for all available SVC algorithms. The tuning process is performed on an artificial classification dataset detailed in Table \ref{tbl:data-details}. The resulting tuned values and their corresponding accuracy are provided in the table.
	\item This is an algorithm selection query which targets all the available classification algorithms in the system trained on the iris dataset. As it can be seen, two SVC algorithm versions, namely A01 and A04, won the selection process by yielding 98\% accuracy.
	\item In this query, our objective is to perform a hybrid algorithm selection and tuning process on the Wine dataset. Specifically, we aim to  tune the $\gamma\sim Uniform(0,1)$ hyperparameter for all the algorithms that have parameter, as well as the \emph{kernel} hyperparameter with value of ``rbf'', and select the algorithm with the highest accuracy. From table \ref{tbl:alg-details}, only the SVC and NuSVC algorithms meet the criteria, which are correctly processed and reported by the proposed mechanism. 
	\item This is a tuning query targetting the $c\sim logUniform(10^{-2},10^{13})$ and $\gamma\sim Uniform(0,1)$ hyperparameters of the available NuSVR algorithms trained the artificial regression dataset in the system.
	\item  This comprehensive selection query encompasses all available regression algorithms. The results reveal that the A10 and A12 algorithms have been successfully identified and selected by the mechanism, as they both yield the same MSE value.
	\item In this query, our aim is to select the best regression algorithm for the California house prices dataset. Specifically, we target the algorithms that possess the $\alpha\sim Uniform(0,10)$ hyperparameter. The selection process will be based on the algorithms' performance after tuning. This example serves to showcase the flexibility of our query-based mechanism.
	\item In this query, we try to determine the sub-optimal values for the DBSCAN algorithm's $\epsilon\sim Uniform(0,10)$ and $metric\in\{\text{cityblock, cosine, euclidean, l1, l2, manhattan}\}$ hyperparameters . The tuning process was carried out using an artificial classification dataset generated by the dedicated data agent assigned to this task.
	\item This query conducts a thorough search across all the available clustering algorithms within the system and selects the one(s) with the highest FM-Score value. 
	\item In a similar fashion to task 6, this query focuses on tuning the $n\_clusters\in\{1,2,\dots, 10\}$ hyperparameter for all compatible clustering algorithms and selects the one that results in the highest FM-Score. Notably, the MBKM algorithm, with a tuned hyperparameter value of 3, has exhibited superior performance compared to all other tuned candidates.  
\end{enumerate}

It is worth highlighting that in all the preceding examples, we have not explicitly specified the exact location of the resources or the types of tasks. The HAMLET system, incorporating the hybrid algorithm selection and tuning mechanism described, effectively handles the exploration and retrieval of the relevant results. The system autonomously manages the process of finding and collecting the appropriate outcomes based on the specified queries and criteria.

\section{Conclusion}\label{sec:conclusion}
This paper presented a hybrid algorithm selection and hyperparameter tuning method that leverages the machine learning resources distributedly organized within the HAMLET platform. The proposed approach utilizes the same query-based technique and core operations to specify, verify, and automatically execute selection and optimization tasks. By being agent-based and built upon the HAMLET platform, our  proposed solution can be easily deployed on a heterogeneous set of devices and allows the implementation of sophisticated collaborative selection and tuning solutions. Furthermore, while our models are generic and independent of any specific ML algorithms, they can effectively utilize computationally limited devices within the HAMLET network when coupled with distributed hyperparameter tuning methodologies, such as the ones in \cite{esmaeili2022hierarchical} and \cite{esmaeili2023agent}. 

To ensure the correctness and analyze the performance of our method, regardless of the way it is implemented or the type of ML resources available on HAMLET network, we provided a series of theoretical claims and proofs and discussed both its time complexity and memory requirements. The presented formal verification proved the termination of the proposed approach with correct answers if they exist and studies how it takes the communication distance and resource accessibility into account. Additionally, the given performance analysis showed that both time and space complexities are linear in terms of the network size, which is negligible compared to the complexity of tasks like tuning and model training in most real cases.

The research opens up various directions for future work. Some suggested avenues include but are not limited to reusability of previous tuning and selection tasks by agents, handling partially or completely duplicate request, enabling horizontal inter-agent connections to improve communication efficiency, investigating and mitigating the effects of unforeseen incidents like agent failures, and incorporating mechanisms to control the presence of potential adversaries.

\backmatter








\bibliography{main}


\begin{thebibliography}{65}
\ifx \bisbn   \undefined \def \bisbn  #1{ISBN #1}\fi
\ifx \binits  \undefined \def \binits#1{#1}\fi
\ifx \bauthor  \undefined \def \bauthor#1{#1}\fi
\ifx \batitle  \undefined \def \batitle#1{#1}\fi
\ifx \bjtitle  \undefined \def \bjtitle#1{#1}\fi
\ifx \bvolume  \undefined \def \bvolume#1{\textbf{#1}}\fi
\ifx \byear  \undefined \def \byear#1{#1}\fi
\ifx \bissue  \undefined \def \bissue#1{#1}\fi
\ifx \bfpage  \undefined \def \bfpage#1{#1}\fi
\ifx \blpage  \undefined \def \blpage #1{#1}\fi
\ifx \burl  \undefined \def \burl#1{\textsf{#1}}\fi
\ifx \doiurl  \undefined \def \doiurl#1{\url{https://doi.org/#1}}\fi
\ifx \betal  \undefined \def \betal{\textit{et al.}}\fi
\ifx \binstitute  \undefined \def \binstitute#1{#1}\fi
\ifx \binstitutionaled  \undefined \def \binstitutionaled#1{#1}\fi
\ifx \bctitle  \undefined \def \bctitle#1{#1}\fi
\ifx \beditor  \undefined \def \beditor#1{#1}\fi
\ifx \bpublisher  \undefined \def \bpublisher#1{#1}\fi
\ifx \bbtitle  \undefined \def \bbtitle#1{#1}\fi
\ifx \bedition  \undefined \def \bedition#1{#1}\fi
\ifx \bseriesno  \undefined \def \bseriesno#1{#1}\fi
\ifx \blocation  \undefined \def \blocation#1{#1}\fi
\ifx \bsertitle  \undefined \def \bsertitle#1{#1}\fi
\ifx \bsnm \undefined \def \bsnm#1{#1}\fi
\ifx \bsuffix \undefined \def \bsuffix#1{#1}\fi
\ifx \bparticle \undefined \def \bparticle#1{#1}\fi
\ifx \barticle \undefined \def \barticle#1{#1}\fi
\bibcommenthead
\ifx \bconfdate \undefined \def \bconfdate #1{#1}\fi
\ifx \botherref \undefined \def \botherref #1{#1}\fi
\ifx \url \undefined \def \url#1{\textsf{#1}}\fi
\ifx \bchapter \undefined \def \bchapter#1{#1}\fi
\ifx \bbook \undefined \def \bbook#1{#1}\fi
\ifx \bcomment \undefined \def \bcomment#1{#1}\fi
\ifx \oauthor \undefined \def \oauthor#1{#1}\fi
\ifx \citeauthoryear \undefined \def \citeauthoryear#1{#1}\fi
\ifx \endbibitem  \undefined \def \endbibitem {}\fi
\ifx \bconflocation  \undefined \def \bconflocation#1{#1}\fi
\ifx \arxivurl  \undefined \def \arxivurl#1{\textsf{#1}}\fi
\csname PreBibitemsHook\endcsname

\bibitem[\protect\citeauthoryear{Esmaeili et~al.}{2022}]{esmaeili2020hamlet}
\begin{barticle}
\bauthor{\bsnm{Esmaeili}, \binits{A.}},
\bauthor{\bsnm{Gallagher}, \binits{J.C.}},
\bauthor{\bsnm{Springer}, \binits{J.A.}},
\bauthor{\bsnm{Matson}, \binits{E.T.}}:
\batitle{Hamlet: A hierarchical agent-based machine learning platform}.
\bjtitle{ACM Trans. Auton. Adapt. Syst.}
(\byear{2022})
\doiurl{10.1145/3530191}
\end{barticle}
\endbibitem

\bibitem[\protect\citeauthoryear{Adankon and Cheriet}{2009}]{adankon2009model}
\begin{barticle}
\bauthor{\bsnm{Adankon}, \binits{M.M.}},
\bauthor{\bsnm{Cheriet}, \binits{M.}}:
\batitle{Model selection for the ls-svm. application to handwriting
  recognition}.
\bjtitle{Pattern Recognition}
\bvolume{42}(\bissue{12}),
\bfpage{3264}--\blpage{3270}
(\byear{2009})
\end{barticle}
\endbibitem

\bibitem[\protect\citeauthoryear{Banerjee et~al.}{2008}]{banerjee2008model}
\begin{barticle}
\bauthor{\bsnm{Banerjee}, \binits{O.}},
\bauthor{\bsnm{El~Ghaoui}, \binits{L.}},
\bauthor{\bsnm{d'Aspremont}, \binits{A.}}:
\batitle{Model selection through sparse maximum likelihood estimation for
  multivariate gaussian or binary data}.
\bjtitle{The Journal of Machine Learning Research}
\bvolume{9},
\bfpage{485}--\blpage{516}
(\byear{2008})
\end{barticle}
\endbibitem

\bibitem[\protect\citeauthoryear{Brazdil et~al.}{2003}]{brazdil2003ranking}
\begin{barticle}
\bauthor{\bsnm{Brazdil}, \binits{P.B.}},
\bauthor{\bsnm{Soares}, \binits{C.}},
\bauthor{\bsnm{Da~Costa}, \binits{J.P.}}:
\batitle{Ranking learning algorithms: Using ibl and meta-learning on accuracy
  and time results}.
\bjtitle{Machine Learning}
\bvolume{50}(\bissue{3}),
\bfpage{251}--\blpage{277}
(\byear{2003})
\end{barticle}
\endbibitem

\bibitem[\protect\citeauthoryear{Chapelle et~al.}{2002}]{chapelle2002model}
\begin{barticle}
\bauthor{\bsnm{Chapelle}, \binits{O.}},
\bauthor{\bsnm{Vapnik}, \binits{V.}},
\bauthor{\bsnm{Bengio}, \binits{Y.}}:
\batitle{Model selection for small sample regression}.
\bjtitle{Machine Learning}
\bvolume{48}(\bissue{1}),
\bfpage{9}--\blpage{23}
(\byear{2002})
\end{barticle}
\endbibitem

\bibitem[\protect\citeauthoryear{Zeng and Luo}{2017}]{zeng2017progressive}
\begin{barticle}
\bauthor{\bsnm{Zeng}, \binits{X.}},
\bauthor{\bsnm{Luo}, \binits{G.}}:
\batitle{Progressive sampling-based bayesian optimization for efficient and
  automatic machine learning model selection}.
\bjtitle{Health information science and systems}
\bvolume{5}(\bissue{1}),
\bfpage{1}--\blpage{21}
(\byear{2017})
\end{barticle}
\endbibitem

\bibitem[\protect\citeauthoryear{You et~al.}{2019}]{you2019towards}
\begin{bchapter}
\bauthor{\bsnm{You}, \binits{K.}},
\bauthor{\bsnm{Wang}, \binits{X.}},
\bauthor{\bsnm{Long}, \binits{M.}},
\bauthor{\bsnm{Jordan}, \binits{M.}}:
\bctitle{Towards accurate model selection in deep unsupervised domain
  adaptation}.
In: \bbtitle{International Conference on Machine Learning},
pp. \bfpage{7124}--\blpage{7133}
(\byear{2019}).
\bcomment{PMLR}
\end{bchapter}
\endbibitem

\bibitem[\protect\citeauthoryear{Bardenet
  et~al.}{2013}]{bardenet2013collaborative}
\begin{bchapter}
\bauthor{\bsnm{Bardenet}, \binits{R.}},
\bauthor{\bsnm{Brendel}, \binits{M.}},
\bauthor{\bsnm{K{\'e}gl}, \binits{B.}},
\bauthor{\bsnm{Sebag}, \binits{M.}}:
\bctitle{Collaborative hyperparameter tuning}.
In: \bbtitle{International Conference on Machine Learning},
pp. \bfpage{199}--\blpage{207}
(\byear{2013}).
\bcomment{PMLR}
\end{bchapter}
\endbibitem

\bibitem[\protect\citeauthoryear{Bengio}{2000}]{bengio2000gradient}
\begin{barticle}
\bauthor{\bsnm{Bengio}, \binits{Y.}}:
\batitle{Gradient-based optimization of hyperparameters}.
\bjtitle{Neural computation}
\bvolume{12}(\bissue{8}),
\bfpage{1889}--\blpage{1900}
(\byear{2000})
\end{barticle}
\endbibitem

\bibitem[\protect\citeauthoryear{Bergstra
  et~al.}{2011}]{bergstra2011algorithms}
\begin{botherref}
\oauthor{\bsnm{Bergstra}, \binits{J.}},
\oauthor{\bsnm{Bardenet}, \binits{R.}},
\oauthor{\bsnm{Bengio}, \binits{Y.}},
\oauthor{\bsnm{K{\'e}gl}, \binits{B.}}:
Algorithms for hyper-parameter optimization.
Advances in neural information processing systems
\textbf{24}
(2011)
\end{botherref}
\endbibitem

\bibitem[\protect\citeauthoryear{Bergstra and
  Bengio}{2012}]{bergstra2012random}
\begin{botherref}
\oauthor{\bsnm{Bergstra}, \binits{J.}},
\oauthor{\bsnm{Bengio}, \binits{Y.}}:
Random search for hyper-parameter optimization.
Journal of machine learning research
\textbf{13}(2)
(2012)
\end{botherref}
\endbibitem

\bibitem[\protect\citeauthoryear{Koch et~al.}{2018}]{koch2018autotune}
\begin{bchapter}
\bauthor{\bsnm{Koch}, \binits{P.}},
\bauthor{\bsnm{Golovidov}, \binits{O.}},
\bauthor{\bsnm{Gardner}, \binits{S.}},
\bauthor{\bsnm{Wujek}, \binits{B.}},
\bauthor{\bsnm{Griffin}, \binits{J.}},
\bauthor{\bsnm{Xu}, \binits{Y.}}:
\bctitle{Autotune: A derivative-free optimization framework for hyperparameter
  tuning}.
In: \bbtitle{Proceedings of the 24th ACM SIGKDD International Conference on
  Knowledge Discovery \& Data Mining},
pp. \bfpage{443}--\blpage{452}
(\byear{2018})
\end{bchapter}
\endbibitem

\bibitem[\protect\citeauthoryear{Maclaurin
  et~al.}{2015}]{maclaurin2015gradient}
\begin{bchapter}
\bauthor{\bsnm{Maclaurin}, \binits{D.}},
\bauthor{\bsnm{Duvenaud}, \binits{D.}},
\bauthor{\bsnm{Adams}, \binits{R.}}:
\bctitle{Gradient-based hyperparameter optimization through reversible
  learning}.
In: \bbtitle{International Conference on Machine Learning},
pp. \bfpage{2113}--\blpage{2122}
(\byear{2015}).
\bcomment{PMLR}
\end{bchapter}
\endbibitem

\bibitem[\protect\citeauthoryear{Parker-Holder
  et~al.}{2020}]{parker2020provably}
\begin{barticle}
\bauthor{\bsnm{Parker-Holder}, \binits{J.}},
\bauthor{\bsnm{Nguyen}, \binits{V.}},
\bauthor{\bsnm{Roberts}, \binits{S.J.}}:
\batitle{Provably efficient online hyperparameter optimization with
  population-based bandits}.
\bjtitle{Advances in Neural Information Processing Systems}
\bvolume{33},
\bfpage{17200}--\blpage{17211}
(\byear{2020})
\end{barticle}
\endbibitem

\bibitem[\protect\citeauthoryear{Thornton et~al.}{2013}]{thornton2013auto}
\begin{bchapter}
\bauthor{\bsnm{Thornton}, \binits{C.}},
\bauthor{\bsnm{Hutter}, \binits{F.}},
\bauthor{\bsnm{Hoos}, \binits{H.H.}},
\bauthor{\bsnm{Leyton-Brown}, \binits{K.}}:
\bctitle{Auto-weka: Combined selection and hyperparameter optimization of
  classification algorithms}.
In: \bbtitle{Proceedings of the 19th ACM SIGKDD International Conference on
  Knowledge Discovery and Data Mining},
pp. \bfpage{847}--\blpage{855}
(\byear{2013})
\end{bchapter}
\endbibitem

\bibitem[\protect\citeauthoryear{Hall et~al.}{2009}]{hall2009weka}
\begin{barticle}
\bauthor{\bsnm{Hall}, \binits{M.}},
\bauthor{\bsnm{Frank}, \binits{E.}},
\bauthor{\bsnm{Holmes}, \binits{G.}},
\bauthor{\bsnm{Pfahringer}, \binits{B.}},
\bauthor{\bsnm{Reutemann}, \binits{P.}},
\bauthor{\bsnm{Witten}, \binits{I.H.}}:
\batitle{The weka data mining software: an update}.
\bjtitle{ACM SIGKDD explorations newsletter}
\bvolume{11}(\bissue{1}),
\bfpage{10}--\blpage{18}
(\byear{2009})
\end{barticle}
\endbibitem

\bibitem[\protect\citeauthoryear{Hutter et~al.}{2011}]{hutter2011sequential}
\begin{bchapter}
\bauthor{\bsnm{Hutter}, \binits{F.}},
\bauthor{\bsnm{Hoos}, \binits{H.H.}},
\bauthor{\bsnm{Leyton-Brown}, \binits{K.}}:
\bctitle{Sequential model-based optimization for general algorithm
  configuration}.
In: \bbtitle{International Conference on Learning and Intelligent
  Optimization},
pp. \bfpage{507}--\blpage{523}
(\byear{2011}).
\bcomment{Springer}
\end{bchapter}
\endbibitem

\bibitem[\protect\citeauthoryear{Bergstra et~al.}{2015}]{bergstra2015hyperopt}
\begin{barticle}
\bauthor{\bsnm{Bergstra}, \binits{J.}},
\bauthor{\bsnm{Komer}, \binits{B.}},
\bauthor{\bsnm{Eliasmith}, \binits{C.}},
\bauthor{\bsnm{Yamins}, \binits{D.}},
\bauthor{\bsnm{Cox}, \binits{D.D.}}:
\batitle{Hyperopt: a python library for model selection and hyperparameter
  optimization}.
\bjtitle{Computational Science \& Discovery}
\bvolume{8}(\bissue{1}),
\bfpage{014008}
(\byear{2015})
\end{barticle}
\endbibitem

\bibitem[\protect\citeauthoryear{Pedregosa et~al.}{2011}]{scikit-learn}
\begin{barticle}
\bauthor{\bsnm{Pedregosa}, \binits{F.}},
\bauthor{\bsnm{Varoquaux}, \binits{G.}},
\bauthor{\bsnm{Gramfort}, \binits{A.}},
\bauthor{\bsnm{Michel}, \binits{V.}},
\bauthor{\bsnm{Thirion}, \binits{B.}},
\bauthor{\bsnm{Grisel}, \binits{O.}},
\bauthor{\bsnm{Blondel}, \binits{M.}},
\bauthor{\bsnm{Prettenhofer}, \binits{P.}},
\bauthor{\bsnm{Weiss}, \binits{R.}},
\bauthor{\bsnm{Dubourg}, \binits{V.}},
\bauthor{\bsnm{Vanderplas}, \binits{J.}},
\bauthor{\bsnm{Passos}, \binits{A.}},
\bauthor{\bsnm{Cournapeau}, \binits{D.}},
\bauthor{\bsnm{Brucher}, \binits{M.}},
\bauthor{\bsnm{Perrot}, \binits{M.}},
\bauthor{\bsnm{Duchesnay}, \binits{E.}}:
\batitle{Scikit-learn: Machine learning in {P}ython}.
\bjtitle{Journal of Machine Learning Research}
\bvolume{12},
\bfpage{2825}--\blpage{2830}
(\byear{2011})
\end{barticle}
\endbibitem

\bibitem[\protect\citeauthoryear{Feurer et~al.}{2015}]{feurer2015efficient}
\begin{botherref}
\oauthor{\bsnm{Feurer}, \binits{M.}},
\oauthor{\bsnm{Klein}, \binits{A.}},
\oauthor{\bsnm{Eggensperger}, \binits{K.}},
\oauthor{\bsnm{Springenberg}, \binits{J.}},
\oauthor{\bsnm{Blum}, \binits{M.}},
\oauthor{\bsnm{Hutter}, \binits{F.}}:
Efficient and robust automated machine learning.
Advances in neural information processing systems
\textbf{28}
(2015)
\end{botherref}
\endbibitem

\bibitem[\protect\citeauthoryear{Guo et~al.}{2019}]{guo2019new}
\begin{bchapter}
\bauthor{\bsnm{Guo}, \binits{X.}},
\bauthor{\bsnm{Stein}, \binits{B.}},
\bauthor{\bsnm{B{\"a}ck}, \binits{T.}}:
\bctitle{A new approach towards the combined algorithm selection and
  hyper-parameter optimization problem}.
In: \bbtitle{2019 IEEE Symposium Series on Computational Intelligence (SSCI)},
pp. \bfpage{2042}--\blpage{2049}
(\byear{2019}).
\bcomment{IEEE}
\end{bchapter}
\endbibitem

\bibitem[\protect\citeauthoryear{van Stein et~al.}{2018}]{van2018automatic}
\begin{botherref}
\oauthor{\bsnm{Stein}, \binits{B.}},
\oauthor{\bsnm{Wang}, \binits{H.}},
\oauthor{\bsnm{B{\"a}ck}, \binits{T.}}:
Automatic configuration of deep neural networks with ego.
arXiv preprint arXiv:1810.05526
(2018)
\end{botherref}
\endbibitem

\bibitem[\protect\citeauthoryear{Czako et~al.}{2021}]{czako2021automaticai}
\begin{barticle}
\bauthor{\bsnm{Czako}, \binits{Z.}},
\bauthor{\bsnm{Sebestyen}, \binits{G.}},
\bauthor{\bsnm{Hangan}, \binits{A.}}:
\batitle{Automaticai--a hybrid approach for automatic artificial intelligence
  algorithm selection and hyperparameter tuning}.
\bjtitle{Expert Systems with Applications}
\bvolume{182},
\bfpage{115225}
(\byear{2021})
\end{barticle}
\endbibitem

\bibitem[\protect\citeauthoryear{Cheng et~al.}{2023}]{cheng2023hpn}
\begin{botherref}
\oauthor{\bsnm{Cheng}, \binits{A.}},
\oauthor{\bsnm{Wang}, \binits{Z.}},
\oauthor{\bsnm{Li}, \binits{Y.}},
\oauthor{\bsnm{Cheng}, \binits{J.}}:
Hpn: Personalized federated hyperparameter optimization.
arXiv preprint arXiv:2304.05195
(2023)
\end{botherref}
\endbibitem

\bibitem[\protect\citeauthoryear{Ryzko}{2020}]{ryzko2020modern}
\begin{bbook}
\bauthor{\bsnm{Ryzko}, \binits{D.}}:
\bbtitle{Modern Big Data Architectures: a Multi-agent Systems Perspective}.
\bpublisher{John Wiley \& Sons}, \blocation{???}
(\byear{2020})
\end{bbook}
\endbibitem

\bibitem[\protect\citeauthoryear{Kargupta et~al.}{1997}]{kargupta1997scalable}
\begin{bchapter}
\bauthor{\bsnm{Kargupta}, \binits{H.}},
\bauthor{\bsnm{Hamzaoglu}, \binits{I.}},
\bauthor{\bsnm{Stafford}, \binits{B.}}:
\bctitle{Scalable, distributed data mining-an agent architecture.}
In: \bbtitle{KDD},
pp. \bfpage{211}--\blpage{214}
(\byear{1997})
\end{bchapter}
\endbibitem

\bibitem[\protect\citeauthoryear{Kargupta
  et~al.}{1999}]{kargupta1999collective}
\begin{barticle}
\bauthor{\bsnm{Kargupta}, \binits{H.}},
\bauthor{\bsnm{Park}, \binits{B.}},
\bauthor{\bsnm{Hershberger}, \binits{D.}},
\bauthor{\bsnm{Johnson}, \binits{E.}}:
\batitle{Collective data mining: A new perspective toward distributed data
  mining}.
\bjtitle{Advances in distributed and parallel knowledge discovery}
\bvolume{2},
\bfpage{131}--\blpage{174}
(\byear{1999})
\end{barticle}
\endbibitem

\bibitem[\protect\citeauthoryear{Gorodetsky et~al.}{2003}]{gorodetsky2003multi}
\begin{bchapter}
\bauthor{\bsnm{Gorodetsky}, \binits{V.}},
\bauthor{\bsnm{Karsaeyv}, \binits{O.}},
\bauthor{\bsnm{Samoilov}, \binits{V.}}:
\bctitle{Multi-agent technology for distributed data mining and
  classification}.
In: \bbtitle{IEEE/WIC International Conference on Intelligent Agent Technology,
  2003. IAT 2003.},
pp. \bfpage{438}--\blpage{441}
(\byear{2003}).
\bcomment{IEEE}
\end{bchapter}
\endbibitem

\bibitem[\protect\citeauthoryear{Albashiri et~al.}{2008}]{albashiri2008emads}
\begin{bchapter}
\bauthor{\bsnm{Albashiri}, \binits{K.A.}},
\bauthor{\bsnm{Coenen}, \binits{F.}},
\bauthor{\bsnm{Leng}, \binits{P.}}:
\bctitle{Emads: An extendible multi-agent data miner}.
In: \bbtitle{International Conference on Innovative Techniques and Applications
  of Artificial Intelligence},
pp. \bfpage{263}--\blpage{275}
(\byear{2008}).
\bcomment{Springer}
\end{bchapter}
\endbibitem

\bibitem[\protect\citeauthoryear{Qasem et~al.}{2021}]{qasem21multi}
\begin{barticle}
\bauthor{\bsnm{Qasem}, \binits{M.H.}},
\bauthor{\bsnm{Obeid}, \binits{N.}},
\bauthor{\bsnm{Hudaib}, \binits{A.}},
\bauthor{\bsnm{Almaiah}, \binits{M.A.}},
\bauthor{\bsnm{Al-Zahrani}, \binits{A.}},
\bauthor{\bsnm{Al-Khasawneh}, \binits{A.}}:
\batitle{Multi-agent system combined with distributed data mining for mutual
  collaboration classification}.
\bjtitle{IEEE Access}
\bvolume{9},
\bfpage{70531}--\blpage{70547}
(\byear{2021})
\doiurl{10.1109/ACCESS.2021.3074125}
\end{barticle}
\endbibitem

\bibitem[\protect\citeauthoryear{Javadpour et~al.}{2023}]{javadpour2023dmaidps}
\begin{barticle}
\bauthor{\bsnm{Javadpour}, \binits{A.}},
\bauthor{\bsnm{Pinto}, \binits{P.}},
\bauthor{\bsnm{Ja’fari}, \binits{F.}},
\bauthor{\bsnm{Zhang}, \binits{W.}}:
\batitle{Dmaidps: a distributed multi-agent intrusion detection and prevention
  system for cloud iot environments}.
\bjtitle{Cluster Computing}
\bvolume{26}(\bissue{1}),
\bfpage{367}--\blpage{384}
(\byear{2023})
\end{barticle}
\endbibitem

\bibitem[\protect\citeauthoryear{Grislin-Le~Strugeon
  et~al.}{2022}]{grislin2022systematic}
\begin{barticle}
\bauthor{\bsnm{Grislin-Le~Strugeon}, \binits{E.}},
\bauthor{\bsnm{Oliveira}, \binits{K.}},
\bauthor{\bsnm{Thilliez}, \binits{M.}},
\bauthor{\bsnm{Petit}, \binits{D.}}:
\batitle{A systematic mapping study on agent mining}.
\bjtitle{Journal of Experimental \& Theoretical Artificial Intelligence}
\bvolume{34}(\bissue{2}),
\bfpage{189}--\blpage{214}
(\byear{2022})
\end{barticle}
\endbibitem

\bibitem[\protect\citeauthoryear{Chemchem et~al.}{2018}]{chemchem2018deep}
\begin{bchapter}
\bauthor{\bsnm{Chemchem}, \binits{A.}},
\bauthor{\bsnm{Alin}, \binits{F.}},
\bauthor{\bsnm{Krajecki}, \binits{M.}}:
\bctitle{Deep learning and data mining classification through the intelligent
  agent reasoning}.
In: \bbtitle{2018 6th International Conference on Future Internet of Things and
  Cloud Workshops (FiCloudW)},
pp. \bfpage{13}--\blpage{20}
(\byear{2018}).
\bcomment{IEEE}
\end{bchapter}
\endbibitem

\bibitem[\protect\citeauthoryear{Yakopcic et~al.}{2019}]{yakopcic2019high}
\begin{bchapter}
\bauthor{\bsnm{Yakopcic}, \binits{C.}},
\bauthor{\bsnm{Rahman}, \binits{N.}},
\bauthor{\bsnm{Atahary}, \binits{T.}},
\bauthor{\bsnm{Taha}, \binits{T.M.}},
\bauthor{\bsnm{Beigh}, \binits{A.}},
\bauthor{\bsnm{Douglass}, \binits{S.}}:
\bctitle{High speed cognitive domain ontologies for asset allocation using
  loihi spiking neurons}.
In: \bbtitle{2019 International Joint Conference on Neural Networks (IJCNN)},
pp. \bfpage{1}--\blpage{8}
(\byear{2019}).
\bcomment{IEEE}
\end{bchapter}
\endbibitem

\bibitem[\protect\citeauthoryear{Iranfar et~al.}{2021}]{iranfar2021multi}
\begin{botherref}
\oauthor{\bsnm{Iranfar}, \binits{A.}},
\oauthor{\bsnm{Zapater}, \binits{M.}},
\oauthor{\bsnm{Atienza}, \binits{D.}}:
Multi-agent reinforcement learning for hyperparameter optimization of
  convolutional neural networks.
IEEE Transactions on Computer-Aided Design of Integrated Circuits and Systems
(2021)
\end{botherref}
\endbibitem

\bibitem[\protect\citeauthoryear{Xue et~al.}{2022}]{xue2022multi}
\begin{barticle}
\bauthor{\bsnm{Xue}, \binits{K.}},
\bauthor{\bsnm{Xu}, \binits{J.}},
\bauthor{\bsnm{Yuan}, \binits{L.}},
\bauthor{\bsnm{Li}, \binits{M.}},
\bauthor{\bsnm{Qian}, \binits{C.}},
\bauthor{\bsnm{Zhang}, \binits{Z.}},
\bauthor{\bsnm{Yu}, \binits{Y.}}:
\batitle{Multi-agent dynamic algorithm configuration}.
\bjtitle{Advances in Neural Information Processing Systems}
\bvolume{35},
\bfpage{20147}--\blpage{20161}
(\byear{2022})
\end{barticle}
\endbibitem

\bibitem[\protect\citeauthoryear{Esmaeili et~al.}{2023}]{esmaeili2023agent}
\begin{barticle}
\bauthor{\bsnm{Esmaeili}, \binits{A.}},
\bauthor{\bsnm{Ghorrati}, \binits{Z.}},
\bauthor{\bsnm{Matson}, \binits{E.T.}}:
\batitle{Agent-based collaborative random search for hyperparameter tuning and
  global function optimization}.
\bjtitle{Systems}
\bvolume{11}(\bissue{5}),
\bfpage{228}
(\byear{2023})
\end{barticle}
\endbibitem

\bibitem[\protect\citeauthoryear{Esmaeili
  et~al.}{2022}]{esmaeili2022hierarchical}
\begin{bchapter}
\bauthor{\bsnm{Esmaeili}, \binits{A.}},
\bauthor{\bsnm{Ghorrati}, \binits{Z.}},
\bauthor{\bsnm{Matson}, \binits{E.T.}}:
\bctitle{Hierarchical collaborative hyper-parameter tuning}.
In: \bbtitle{Advances in Practical Applications of Agents, Multi-Agent Systems,
  and Complex Systems Simulation. The PAAMS Collection: 20th International
  Conference, PAAMS 2022, L'Aquila, Italy, July 13--15, 2022, Proceedings},
pp. \bfpage{127}--\blpage{139}
(\byear{2022}).
\bcomment{Springer}
\end{bchapter}
\endbibitem

\bibitem[\protect\citeauthoryear{Ben-Hur et~al.}{2001}]{ben2001stability}
\begin{bchapter}
\bauthor{\bsnm{Ben-Hur}, \binits{A.}},
\bauthor{\bsnm{Elisseeff}, \binits{A.}},
\bauthor{\bsnm{Guyon}, \binits{I.}}:
\bctitle{A stability based method for discovering structure in clustered data}.
In: \bbtitle{{Biocomputing 2002}},
pp. \bfpage{6}--\blpage{17}.
\bpublisher{World Scientific}, \blocation{???}
(\byear{2001})
\end{bchapter}
\endbibitem

\bibitem[\protect\citeauthoryear{Kohavi et~al.}{1995}]{kohavi1995study}
\begin{bchapter}
\bauthor{\bsnm{Kohavi}, \binits{R.}}, \betal:
\bctitle{A study of cross-validation and bootstrap for accuracy estimation and
  model selection}.
In: \bbtitle{Ijcai},
vol. \bseriesno{14},
pp. \bfpage{1137}--\blpage{1145}
(\byear{1995}).
\bcomment{Montreal, Canada}
\end{bchapter}
\endbibitem

\bibitem[\protect\citeauthoryear{Raschka}{2018}]{raschka2018model}
\begin{botherref}
\oauthor{\bsnm{Raschka}, \binits{S.}}:
Model evaluation, model selection, and algorithm selection in machine learning.
arXiv preprint arXiv:1811.12808
(2018)
\end{botherref}
\endbibitem

\bibitem[\protect\citeauthoryear{Tibshirani and
  Walther}{2005}]{tibshirani2005cluster}
\begin{barticle}
\bauthor{\bsnm{Tibshirani}, \binits{R.}},
\bauthor{\bsnm{Walther}, \binits{G.}}:
\batitle{Cluster validation by prediction strength}.
\bjtitle{Journal of Computational and Graphical Statistics}
\bvolume{14}(\bissue{3}),
\bfpage{511}--\blpage{528}
(\byear{2005})
\end{barticle}
\endbibitem

\bibitem[\protect\citeauthoryear{Fischer et~al.}{2003}]{fischer2003holonic}
\begin{bchapter}
\bauthor{\bsnm{Fischer}, \binits{K.}},
\bauthor{\bsnm{Schillo}, \binits{M.}},
\bauthor{\bsnm{Siekmann}, \binits{J.}}:
\bctitle{Holonic multiagent systems: A foundation for the organisation of
  multiagent systems}.
In: \bbtitle{International Conference on Industrial Applications of Holonic and
  Multi-agent Systems},
pp. \bfpage{71}--\blpage{80}
(\byear{2003}).
\bcomment{Springer}
\end{bchapter}
\endbibitem

\bibitem[\protect\citeauthoryear{Harrison~Jr and
  Rubinfeld}{1978}]{harrison1978hedonic}
\begin{botherref}
\oauthor{\bsnm{Harrison~Jr}, \binits{D.}},
\oauthor{\bsnm{Rubinfeld}, \binits{D.L.}}:
Hedonic housing prices and the demand for clean air
(1978)
\end{botherref}
\endbibitem

\bibitem[\protect\citeauthoryear{{Kelley Pace} and
  Barry}{1997}]{KELLEYPACE1997291}
\begin{barticle}
\bauthor{\bsnm{{Kelley Pace}}, \binits{R.}},
\bauthor{\bsnm{Barry}, \binits{R.}}:
\batitle{Sparse spatial autoregressions}.
\bjtitle{Statistics and Probability Letters}
\bvolume{33}(\bissue{3}),
\bfpage{291}--\blpage{297}
(\byear{1997})
\end{barticle}
\endbibitem

\bibitem[\protect\citeauthoryear{Fisher}{1936}]{fisher1936use}
\begin{barticle}
\bauthor{\bsnm{Fisher}, \binits{R.A.}}:
\batitle{The use of multiple measurements in taxonomic problems}.
\bjtitle{Annals of eugenics}
\bvolume{7}(\bissue{2}),
\bfpage{179}--\blpage{188}
(\byear{1936})
\end{barticle}
\endbibitem

\bibitem[\protect\citeauthoryear{Lichman et~al.}{2013}]{lichman2013uci}
\begin{botherref}
\oauthor{\bsnm{Lichman}, \binits{M.}}, et al.:
UCI machine learning repository.
Irvine, CA
(2013)
\end{botherref}
\endbibitem

\bibitem[\protect\citeauthoryear{Wolberg et~al.}{1994}]{wolberg1994machine}
\begin{barticle}
\bauthor{\bsnm{Wolberg}, \binits{W.H.}},
\bauthor{\bsnm{Street}, \binits{W.N.}},
\bauthor{\bsnm{Mangasarian}, \binits{O.L.}}:
\batitle{Machine learning techniques to diagnose breast cancer from
  image-processed nuclear features of fine needle aspirates}.
\bjtitle{Cancer letters}
\bvolume{77}(\bissue{2-3}),
\bfpage{163}--\blpage{171}
(\byear{1994})
\end{barticle}
\endbibitem

\bibitem[\protect\citeauthoryear{Alpaydin and
  Kaynak}{1998}]{alpaydin1998cascading}
\begin{barticle}
\bauthor{\bsnm{Alpaydin}, \binits{E.}},
\bauthor{\bsnm{Kaynak}, \binits{C.}}:
\batitle{Cascading classifiers}.
\bjtitle{Kybernetika}
\bvolume{34}(\bissue{4}),
\bfpage{369}--\blpage{374}
(\byear{1998})
\end{barticle}
\endbibitem

\bibitem[\protect\citeauthoryear{Efron et~al.}{2004}]{efron2004least}
\begin{barticle}
\bauthor{\bsnm{Efron}, \binits{B.}},
\bauthor{\bsnm{Hastie}, \binits{T.}},
\bauthor{\bsnm{Johnstone}, \binits{I.}},
\bauthor{\bsnm{Tibshirani}, \binits{R.}}, \betal:
\batitle{Least angle regression}.
\bjtitle{The Annals of statistics}
\bvolume{32}(\bissue{2}),
\bfpage{407}--\blpage{499}
(\byear{2004})
\end{barticle}
\endbibitem

\bibitem[\protect\citeauthoryear{}{}]{scikit-learn-web}
\begin{botherref}
{Scikit-learn} API Reference.
\url{https://scikit-learn.org/stable/modules/classes.html}.
Accessed: 2020-05-20
\end{botherref}
\endbibitem

\bibitem[\protect\citeauthoryear{Chang and Lin}{2011}]{chang2011libsvm}
\begin{barticle}
\bauthor{\bsnm{Chang}, \binits{C.-C.}},
\bauthor{\bsnm{Lin}, \binits{C.-J.}}:
\batitle{Libsvm: a library for support vector machines}.
\bjtitle{ACM transactions on intelligent systems and technology (TIST)}
\bvolume{2}(\bissue{3}),
\bfpage{1}--\blpage{27}
(\byear{2011})
\end{barticle}
\endbibitem

\bibitem[\protect\citeauthoryear{Rennie et~al.}{2003}]{rennie2003tackling}
\begin{bchapter}
\bauthor{\bsnm{Rennie}, \binits{J.D.}},
\bauthor{\bsnm{Shih}, \binits{L.}},
\bauthor{\bsnm{Teevan}, \binits{J.}},
\bauthor{\bsnm{Karger}, \binits{D.R.}}:
\bctitle{Tackling the poor assumptions of naive bayes text classifiers}.
In: \bbtitle{Proceedings of the 20th International Conference on Machine
  Learning (ICML-03)},
pp. \bfpage{616}--\blpage{623}
(\byear{2003})
\end{bchapter}
\endbibitem

\bibitem[\protect\citeauthoryear{Hastie et~al.}{2009}]{hastie2009elements}
\begin{bbook}
\bauthor{\bsnm{Hastie}, \binits{T.}},
\bauthor{\bsnm{Tibshirani}, \binits{R.}},
\bauthor{\bsnm{Friedman}, \binits{J.}}:
\bbtitle{The Elements of Statistical Learning: Data Mining, Inference, and
  Prediction}.
\bpublisher{Springer}, \blocation{???}
(\byear{2009})
\end{bbook}
\endbibitem

\bibitem[\protect\citeauthoryear{Tibshirani
  et~al.}{2002}]{tibshirani2002diagnosis}
\begin{barticle}
\bauthor{\bsnm{Tibshirani}, \binits{R.}},
\bauthor{\bsnm{Hastie}, \binits{T.}},
\bauthor{\bsnm{Narasimhan}, \binits{B.}},
\bauthor{\bsnm{Chu}, \binits{G.}}:
\batitle{Diagnosis of multiple cancer types by shrunken centroids of gene
  expression}.
\bjtitle{Proceedings of the National Academy of Sciences}
\bvolume{99}(\bissue{10}),
\bfpage{6567}--\blpage{6572}
(\byear{2002})
\end{barticle}
\endbibitem

\bibitem[\protect\citeauthoryear{Hoerl and Kennard}{1970}]{hoerl1970ridge}
\begin{barticle}
\bauthor{\bsnm{Hoerl}, \binits{A.E.}},
\bauthor{\bsnm{Kennard}, \binits{R.W.}}:
\batitle{Ridge regression: Biased estimation for nonorthogonal problems}.
\bjtitle{Technometrics}
\bvolume{12}(\bissue{1}),
\bfpage{55}--\blpage{67}
(\byear{1970})
\end{barticle}
\endbibitem

\bibitem[\protect\citeauthoryear{Murphy}{2012}]{murphy2012machine}
\begin{bbook}
\bauthor{\bsnm{Murphy}, \binits{K.P.}}:
\bbtitle{Machine Learning: a Probabilistic Perspective}.
\bpublisher{MIT press}, \blocation{???}
(\byear{2012})
\end{bbook}
\endbibitem

\bibitem[\protect\citeauthoryear{Tibshirani}{1996}]{tibshirani1996regression}
\begin{barticle}
\bauthor{\bsnm{Tibshirani}, \binits{R.}}:
\batitle{Regression shrinkage and selection via the lasso}.
\bjtitle{Journal of the Royal Statistical Society: Series B (Methodological)}
\bvolume{58}(\bissue{1}),
\bfpage{267}--\blpage{288}
(\byear{1996})
\end{barticle}
\endbibitem

\bibitem[\protect\citeauthoryear{Zou and Hastie}{2005}]{zou2005regularization}
\begin{barticle}
\bauthor{\bsnm{Zou}, \binits{H.}},
\bauthor{\bsnm{Hastie}, \binits{T.}}:
\batitle{Regularization and variable selection via the elastic net}.
\bjtitle{Journal of the royal statistical society: series B (statistical
  methodology)}
\bvolume{67}(\bissue{2}),
\bfpage{301}--\blpage{320}
(\byear{2005})
\end{barticle}
\endbibitem

\bibitem[\protect\citeauthoryear{Lloyd}{1982}]{lloyd1982least}
\begin{barticle}
\bauthor{\bsnm{Lloyd}, \binits{S.}}:
\batitle{Least squares quantization in pcm}.
\bjtitle{IEEE transactions on information theory}
\bvolume{28}(\bissue{2}),
\bfpage{129}--\blpage{137}
(\byear{1982})
\end{barticle}
\endbibitem

\bibitem[\protect\citeauthoryear{Sculley}{2010}]{sculley2010web}
\begin{bchapter}
\bauthor{\bsnm{Sculley}, \binits{D.}}:
\bctitle{Web-scale k-means clustering}.
In: \bbtitle{Proceedings of the 19th International Conference on World Wide
  Web},
pp. \bfpage{1177}--\blpage{1178}
(\byear{2010})
\end{bchapter}
\endbibitem

\bibitem[\protect\citeauthoryear{Ester et~al.}{1996}]{ester1996density}
\begin{bchapter}
\bauthor{\bsnm{Ester}, \binits{M.}},
\bauthor{\bsnm{Kriegel}, \binits{H.-P.}},
\bauthor{\bsnm{Sander}, \binits{J.}},
\bauthor{\bsnm{Xu}, \binits{X.}}, \betal:
\bctitle{A density-based algorithm for discovering clusters in large spatial
  databases with noise.}
In: \bbtitle{Kdd},
vol. \bseriesno{96},
pp. \bfpage{226}--\blpage{231}
(\byear{1996})
\end{bchapter}
\endbibitem

\bibitem[\protect\citeauthoryear{Zhang et~al.}{1996}]{zhang1996birch}
\begin{barticle}
\bauthor{\bsnm{Zhang}, \binits{T.}},
\bauthor{\bsnm{Ramakrishnan}, \binits{R.}},
\bauthor{\bsnm{Livny}, \binits{M.}}:
\batitle{Birch: an efficient data clustering method for very large databases}.
\bjtitle{ACM Sigmod Record}
\bvolume{25}(\bissue{2}),
\bfpage{103}--\blpage{114}
(\byear{1996})
\end{barticle}
\endbibitem

\bibitem[\protect\citeauthoryear{Rokach and
  Maimon}{2005}]{rokach2005clustering}
\begin{bchapter}
\bauthor{\bsnm{Rokach}, \binits{L.}},
\bauthor{\bsnm{Maimon}, \binits{O.}}:
\bctitle{Clustering methods}.
In: \bbtitle{Data Mining and Knowledge Discovery Handbook},
pp. \bfpage{321}--\blpage{352}.
\bpublisher{Springer}, \blocation{???}
(\byear{2005})
\end{bchapter}
\endbibitem

\bibitem[\protect\citeauthoryear{Fowlkes and Mallows}{1983}]{fowlkes1983method}
\begin{barticle}
\bauthor{\bsnm{Fowlkes}, \binits{E.B.}},
\bauthor{\bsnm{Mallows}, \binits{C.L.}}:
\batitle{A method for comparing two hierarchical clusterings}.
\bjtitle{Journal of the American statistical association}
\bvolume{78}(\bissue{383}),
\bfpage{553}--\blpage{569}
(\byear{1983})
\end{barticle}
\endbibitem

\end{thebibliography}

\end{document}